\documentclass{article}
\usepackage{etoolbox}
\newtoggle{neurips}
\togglefalse{neurips}

\newcommand{\neurips}[1]{\iftoggle{neurips}{#1}{}}
\newcommand{\arxiv}[1]{\iftoggle{neurips}{}{#1}}

\neurips{\usepackage[nonatbib]{neurips_2021}
}
\arxiv{\usepackage[letterpaper, left=1in, right=1in, top=1in, bottom=1in]{geometry}}

\newcommand{\citep}{\cite}
\newcommand{\citet}{\cite}
\usepackage{amsmath,amsthm,amssymb}
\usepackage{bbm}

\neurips{
\usepackage[utf8]{inputenc} 
\usepackage[T1]{fontenc}    
\usepackage{url}            
\usepackage{amsfonts}       
\usepackage{nicefrac}       
\usepackage{microtype}      
\usepackage{graphicx}
}
\usepackage{booktabs}       
\usepackage{xcolor}         
\usepackage{color}
\usepackage{xspace}
\definecolor{dgreen}{rgb}{0,0.5,0}
\usepackage[colorlinks = true,
            linkcolor = blue,
            urlcolor  = blue,
            citecolor = dgreen,
            anchorcolor = blue]{hyperref}
\usepackage{enumitem}            
\usepackage{makecell}
\usepackage{adjustbox}

\makeatletter
\newtheorem*{rep@theorem}{\rep@title}
\newcommand{\newreptheorem}[2]{%
\newenvironment{rep#1}[1]{%
 \def\rep@title{#2 \ref{##1}}%
 \begin{rep@theorem}}%
 {\end{rep@theorem}}}
\makeatother

\newtheorem{theorem}{Theorem}[section]
\newtheorem{corollary}[theorem]{Corollary}
\newtheorem{lemma}[theorem]{Lemma}
\newtheorem{claim}[theorem]{Claim}
\newtheorem{remark}[theorem]{Remark}
\newtheorem{fact}[theorem]{Fact}
\newreptheorem{lemma}{Lemma}
\newreptheorem{theorem}{Theorem}

\theoremstyle{definition}
\newtheorem{defn}{Definition}[section]

\newcommand{\nc}{\newcommand}
\newcount\Comments
\nc\noah[1]{\ifnum\Comments=1 {\textcolor{purple}{[ng: #1]}}\fi}
\nc\maxfish[1]{\ifnum\Comments=1{\textcolor{blue}{[mf: #1]}}\fi}
\nc\costis[1]{\ifnum\Comments=1{\textcolor{brown}{[cd: #1]}}\fi}
\nc\costiss[1]{\textcolor{red}{#1}}

\nc{\Opthedge}{Optimistic Hedge\xspace}

\nc{\DMO}{\DeclareMathOperator}
\nc\old[1]{\textcolor{brown}{[old: #1]}}
\nc{\BR}{\mathbb{R}}
\nc{\BC}{\mathbb{C}}
\DMO{\Bin}{Bin}
\nc{\BN}{\mathbb{N}}
\nc{\distrs}[1]{\Delta({#1})}
\nc{\BZ}{\mathbb{Z}}
\nc{\ep}{\epsilon}
\nc{\ra}{\rightarrow}
\nc{\st}{\star}
\DMO{\REG}{Reg}
\nc{\Reg}[2]{\REG_{{#1},{#2}}}
\nc{\til}{\tilde}
\nc{\kld}[2]{\KL({#1};{#2})}
\nc{\chisq}[2]{\chi^2({#1};{#2})}
\DMO{\POLYLOG}{polylog}
\nc{\polylog}{\POLYLOG}
\renewcommand{\t}{\top}
\nc{\matx}[1]{\left(\begin{matrix}#1\end{matrix}\right)}
\DMO{\VAR}{Var}
\DMO{\COV}{Cov}
\nc{\Var}[2]{\VAR_{{#1}}\left({#2}\right)}
\nc{\Cov}[3]{\COV_{{#1}}\left({#2},{#3}\right)}
\DMO{\DD}{D}
\nc{\fd}[2]{\DD_{#1}{#2}}
\nc{\fds}[3]{\left(\fd{#1}{#2}\right)\^{#3}}
\nc{\fdc}[2]{\DD^\circ_{#1}{#2}}
\nc{\fdcs}[3]{\left(\fdc{#1}{#2}\right)\^{#3}}
\DMO{\EEE}{E}
\nc{\shf}[2]{\EEE_{#1}{#2}}
\nc{\shfs}[3]{\left(\shf{#1}{#2}\right)\^{#3}}
\nc{\norm}[2]{\left\| {#2} \right\|_{#1}}
\nc{\normst}[2]{\left\| {#2} \right\|_{#1}^\st}
\renewcommand{\^}[1]{^{(#1)}}

\DeclareMathOperator*{\argmin}{arg\,min}

\nc{\grad}{\nabla}
\nc{\lng}{\langle}
\nc{\rng}{\rangle}
\nc{\bbone}{\mathbf{1}}
\nc{\bbzero}{\mathbf{0}}
\nc{\MD}{\mathcal{D}}
\nc{\MM}{\mathcal{M}}
\nc{\MZ}{\mathcal{Z}}
\nc{\MU}{\mathcal{U}}
\nc{\MP}{\mathcal{P}}
\nc{\MC}{\mathcal{C}}
\nc{\MT}{\mathcal{T}}
\nc{\MS}{\mathcal{S}}
\nc{\MX}{\mathcal{X}}
\nc{\MY}{\mathcal{Y}}
\nc{\MA}{\mathcal{A}}
\nc{\MB}{\mathcal{B}}
\nc{\MJ}{\mathcal{J}}
\nc{\MF}{\mathcal{F}}
\nc{\MG}{\mathcal{G}}
\nc{\MR}{\mathcal{R}}
\nc{\ML}{\mathcal{L}}
\nc{\MQ}{\mathcal{Q}}
\newcommand{\p}[1]{\left({#1}\right)}
\nc{\E}{\mathbb{E}}

\nc{\ba}{\mathbf{A}}
\nc{\bx}{\mathbf{x}}
\nc{\by}{\mathbf{y}}
\nc{\bz}{\mathbf{z}}
\nc{\bs}{\mathbf{s}}
\nc{\bt}{\mathbf{t}}
\nc{\br}{\mathbf{r}}

\nc{\ME}{\mathcal{E}}
\DMO{\View}{View}
\DMO{\KL}{KL}
\nc{\MW}{\mathcal{W}}
\nc{\CS}{\mathscr{S}}
\nc{\CI}{\mathscr{I}}
\nc{\CQ}{\mathscr{Q}}
\nc{\CL}{\mathscr{L}}
\nc{\CM}{\mathscr{M}}
\nc{\CG}{\mathscr{G}}
\nc{\CR}{\mathscr{R}}

\nc{\wh}{\widehat}

\newcommand{\ps}[1]{\left[{#1}\right]}
\newcommand{\set}[1]{\left\{#1\right\}}
\newcommand{\card}[1]{\left|#1\right|}

\newcommand{\BO}{\mathbbm{1}}
\newcommand{\cU}{\mathcal{U}}
\newcommand{\cF}{\mathcal{F}}
\newcommand{\EE}{\mathbb{E}}

\title{Near-Optimal No-Regret Learning in General Games}

\author{%
  Constantinos Daskalakis\thanks{Supported by NSF Awards CCF-1901292,  DMS-2022448 and  DMS-2134108, by a Simons Investigator Award, by the Simons Collaboration on the Theory of Algorithmic Fairness, by a DSTA grant, and by the DOE PhILMs project (No. DE-AC05-76RL01830).} \\
  MIT CSAIL\\
  \url{costis@csail.mit.edu}  \and
  Maxwell Fishelson\\
  MIT CSAIL\\
  \url{maxfish@mit.edu}  \and
  Noah Golowich\thanks{Supported by a Fannie \& John Hertz Foundation Fellowship and an NSF Graduate Fellowship.}\\
  MIT CSAIL\\
  \url{nzg@mit.edu}\\
}

\begin{document}

\maketitle

\begin{abstract}
  We show that Optimistic Hedge -- a common variant of multiplicative-weights-updates with recency bias -- attains ${\rm poly}(\log T)$ regret in multi-player general-sum games. In particular, when every player of the game uses Optimistic Hedge to iteratively update her strategy in response to the history of play so far, then after $T$ rounds of interaction, {\em each player} experiences total regret that is ${\rm poly}(\log T)$. Our bound improves, exponentially, the $O({T}^{1/2})$ regret attainable  by standard no-regret learners in games, the $O(T^{1/4})$ regret attainable by no-regret learners with recency bias~\citep{syrgkanis_fast_2015}, and the ${O}(T^{1/6})$ bound that was recently shown for Optimistic Hedge in the special case of two-player games~\citep{chen_hedging_2020}. 
  A corollary of our bound is that Optimistic Hedge converges to coarse correlated equilibrium in general games at a rate of $\tilde{O}\left(\frac 1T\right)$.

\end{abstract}

\section{Introduction} \label{sec:intro}

\neurips{\vspace{-0.3cm}}
Online learning has a long history that is intimately related to the development of game theory, convex optimization, and machine learning. One of its earliest instantiations can be traced to~Brown's proposal \citet{Brown1949} of fictitious play as a method to solve two-player zero-sum games. Indeed, as  shown by~\citet{robinson1951iterative}, when the players of (zero-sum) matrix game
use fictitious play to iteratively update their actions in response to each other's history of play, the resulting dynamics converge in the following sense: the product of the empirical distributions of strategies for each player converges to the set of Nash equilibria in the game, 
though the rate of convergence is now known to be exponentially slow~\citep{daskalakis2014counter}. 
Moreover, such convergence to Nash equilibria fails in non-zero-sum games \citep{shapley_topics_1963}.

The slow convergence of fictitious play to Nash equilibria in zero-sum matrix games and non-convergence in general-sum games can be mitigated by appealing to the pioneering works \citet{blackwell1954controlled,hannan1957approximation} and the ensuing literature on no-regret learning~\citep{cesa2006prediction}.
It is known that if both players of a zero-sum matrix game experience regret that is at most $\varepsilon(T)$, the product of the players' empirical distributions of strategies is an $O(\varepsilon(T)/T)$-approximate Nash equilibrium. 
More generally, if each player of a general-sum, multi-player game experiences regret that is at most~$ \varepsilon(T)$,  the empirical distribution of joint strategies  converges to a coarse correlated equilibrium\footnote{In general-sum games, it is typical to focus on proving convergence rates for weaker types of equilibrium than Nash, such as coarse correlated equilibria, since finding Nash equilibria is PPAD-complete \citep{daskalakis_complexity_2006,chen_settling_2009}.} of the game, at a rate of~$O(\varepsilon(T)/{T})$. Importantly, a multitude of online learning algorithms, such as the celebrated Hedge and Follow-The-Perturbed-Leader algorithms, guarantee adversarial 
regret $O(\sqrt{T})$~\citep{cesa2006prediction}. Thus, when such algorithms are employed by all players in a game, their $O(\sqrt{T})$ regret  implies convergence to coarse correlated equilibria (and Nash equilibria of matrix games) at a rate of $O(1/\sqrt{T})$.

While standard no-regret learners guarantee $O(\sqrt{T})$ regret for each player in a game, the players can do better by employing specialized no-regret learning procedures. Indeed, it was established by~\citet{daskalakis_near-optimal_2013} 
that there exists a somewhat complex no-regret learner based on Nesterov's excessive gap technique \cite{nesterov2005excessive}, which guarantees $O(\log T)$ regret to each player of a two-player zero-sum game. This represents an exponential improvement over the regret guaranteed by standard no-regret learners. More generally, \citet{syrgkanis_fast_2015} established that if players of a multi-player, general-sum game use any algorithm from the family of Optimistic Mirror Descent (MD) or Optimistic Follow-the-Regularized-Leader (FTRL) algorithms (which are analogoues of the MD and FTRL algorithms, respectively, with recency bias), each player enjoys regret that is $O(T^{1/4})$. This was recently improved by~\citet{chen_hedging_2020} to $O(T^{1/6})$ in the special case of two-player games in which the players use \Opthedge, a particularly simple representative from both the Optimistic MD and Optimistic FTRL families.

The above results for general-sum games represent significant improvements over the $O(\sqrt{T})$ regret attainable by standard no-regret learners, but are not as dramatic as the logarithmic regret that has been shown attainable by no-regret learners, albeit more complex ones, in 2-player zero-sum games (e.g., \citet{daskalakis_near-optimal_2013}). {Indeed, despite extensive work on no-regret learning, understanding the optimal regret that can be guaranteed by no-regret learning algorithms in general-sum games has remained elusive. This question is especially intruiging in light of experiments suggesting that polylogarithmic regret should be attainable \citep{syrgkanis_fast_2015,hsieh_adaptive_2021}. 
}
In this paper we settle this question by showing that no-regret learners can guarantee polylogarithmic regret to each player in general-sum multi-player games. 
Moreover, this regret is attainable by a particularly simple algorithm -- Optimistic Hedge: 

\begin{table}
  \centering
  \caption{Overview of prior work on fast rates for learning in games. $m$ denotes the number of players, and $n$ denotes the number of actions per player (assumed to be the same for all players). 
    For \Opthedge, the adversarial regret bounds in the right-hand column are obtained via a choice of adaptive step-sizes. The $\tilde O(\cdot)$ notation hides factors that are polynomial in $\log T$.}
  \label{tab:prior-work}
  \begin{adjustbox}{center}
  \begin{tabular}{lccc}
    \toprule
    Algorithm & Setting & Regret in games & Adversarial regret \\
    \midrule
    \makecell[l]{Hedge   (\& many\\ other algs.)} & \makecell[c]{multi-player,\\ general-sum} & \makecell[c]{$O(\sqrt{T \log n})$   \citep{cesa2006prediction}} & \makecell[c]{$O(\sqrt{T \log n})$  \citep{cesa2006prediction}} \\ 
    \midrule 
    \makecell[l]{Excessive Gap \\ Technique } &\makecell[c]{ 2-player,\\ 0-sum} & \makecell[c]{$O(\log n(\log T + \log^{3/2} n))$ \\ \citet{daskalakis_near-optimal_2013}} & \makecell[c]{$O(\sqrt{T \log n})$ \\ \citet{daskalakis_near-optimal_2013}} \\
     \midrule
    \makecell[l]{ DS-OptMD, OptDA} & \makecell[c]{2-player, 0-sum} &  \makecell[c]{$\log^{O(1)} (n)$  \citet{hsieh_adaptive_2021}} & \makecell[c]{$\sqrt{T \log^{O(1)} (n)}$  \citet{hsieh_adaptive_2021}} \\
    \midrule
   \Opthedge & \makecell[c]{multi-player,\\ general-sum} 
               & \makecell[c]{$O(\log n \cdot \sqrt{m} \cdot T^{1/4})$ \\  \citet{rakhlin_optimization_2013,syrgkanis_fast_2015}}
              & \makecell[c]{$\tilde O(\sqrt{T \log n})$  \\  \citet{rakhlin_optimization_2013,syrgkanis_fast_2015}} \\
    \midrule
   \Opthedge & \makecell[c]{2-player,\\ general-sum} 
               & \makecell[c]{$O(\log^{5/6}n \cdot T^{1/6})$  \citet{chen_hedging_2020} }
              & \makecell[c]{$\tilde O(\sqrt{T \log n})$ } \\
    \midrule
   \Opthedge & \makecell[c]{multi-player,\\ general-sum} 
               & \makecell[c]{$O(\log n \cdot m \cdot \log^4 T)$ \\ {\bf (Theorem \ref{thm:polylog-main}) }}
              & \makecell[c]{$\tilde O(\sqrt{T \log n})$ \\ {\bf (Corollary \ref{cor:polylog-adv})}} \\    
    \bottomrule
  \end{tabular}
  \end{adjustbox}
\end{table}
\begin{theorem}[Abbreviated version of Theorem \ref{thm:polylog-main}] \label{thm:main}
  Suppose that $m$ players play a general-sum multi-player game, with a finite set of $n$ strategies per player, over $T$ rounds. Suppose also that each player uses Optimistic Hedge to update her strategy in every round, as a function of the history of play so far. Then each player experiences $O(m \cdot \log n \cdot \log^4 T)$
  regret. 
\end{theorem}
An immediate corollary of Theorem \ref{thm:main} is that the empirical distribution of play is a $O\left(\frac{m \log n \log^4 T}{T}\right)$-approximate coarse correlated equilibrium (CCE) of the game. 
We remark that Theorem \ref{thm:main} bounds the total regret experienced by {\em each player} of the multi-player game, which is the most standard regret objective for no-regret learning in games, and which is essential to achieve convergence to CCE. For the looser objective of the {\em average} of all players' regrets, \cite{rakhlin_optimization_2013} established  a $O(\log n)$ bound for \Opthedge in two-player zero-sum games, and \cite{syrgkanis_fast_2015} generalized  this bound, to $O(m \log n)$ in $m$-player general-sum games. Note that since some players may experience \emph{negative regret} \cite{hsieh_adaptive_2021}, the average of the players' regrets cannot be used in general to bound the maximum regret experienced by any individual player. 
Finally, we remark that several results in the literature posit no-regret learning as a model of agents' rational behavior; for instance, \cite{roughgarden_intrinsic_2009,syrgkanis_composable_2013,roughgarden_price_2017} show that no-regret learners in smooth games enjoy strong Price-of-Anarchy bounds. {By showing that \emph{each agent} can obtain very small regret in games by playing \Opthedge, Theorem \ref{thm:main} strengthens the plausability of the common assumption made in this literature that each agent will choose to use such a no-regret algorithm.}
  
\subsection{Related work}

\neurips{\vspace{-0.2cm}}
Table \ref{tab:prior-work} summarizes the prior works that aim to establish optimal regret bounds for no-regret learners in games. We remark that \cite{chen_hedging_2020} shows that the regret of Hedge 
is $\Omega(\sqrt{T})$ even in 2-player games where each player has 2 actions, meaning that optimism is necessary to obtain fast rates. 
 The table also includes a recent result of \cite{hsieh_adaptive_2021} showing that when the players in a 2-player zero-sum game with $n$ actions per player use a variant of \Opthedge with adaptive step size (a special case of their algorithms DS-OptMD and OptDA), each player has $\log^{O(1)} n$ regret. 
The techniques of \cite{hsieh_adaptive_2021} differ substantially from ours: the result in \cite{hsieh_adaptive_2021} is based on showing that the joint strategies $x\^t$ rapidly converge, pointwise, to a Nash equilibrium $x^\st$. 
Such a result seems very unlikely to extend to our setting of general-sum games, since finding an approximate Nash equilibrium even in 2-player games is PPAD-complete \cite{chen_settling_2009}. 
We also remark that the earlier work \citep{kangarshahi_honest_2018} shows that each player's regret is at most  $O(\log T \cdot \log n)$ when they use a certain algorithm based on Optimistic MD in 2-player zero-sum games; their technique is heavily tailored to 2-player zero-sum games, relying on the notion of duality in such a setting.

\cite{foster_learning_2016} shows that one can obtain fast rates in games for a broader class of algorithms (e.g., including Hedge) if one adopts a relaxed (approximate) notion of optimality. \cite{wei_more_2018} uses optimism to obtain adaptive regret bounds for bandit problems. Many recent papers (e.g., \citep{daskalakis_last-iterate_2019,golowich_tight_2020,lei_last_2021,hsieh_adaptive_2021,wei_linear_2021,azizian_last-iterate_2021}) have studied the \emph{last-iterate} convergence of algorithms from the Optimistic Mirror Descent family, which includes \Opthedge. Finally, a long line of papers (e.g., \citep{hart_uncoupled_2003,daskalakis_learning_2010,kleinberg_beyond_2011,balcan_weighted_2012,papadimitriou_nash_2016,bailey_multiplicative_2018,mertikopoulos_cycles_2018,bailey_fast_2019,cheung_vortices_2019,vlatakis_no-regret_2020}) has studied the dynamics of learning algorithms in games. Essentially all of these papers do not use optimism, and many of them show \emph{non-convergence} (e.g., divergence or recurrence) of the iterates of various learning algorithms such as FTRL and Mirror Descent when used in games.

\section{Preliminaries}
\label{sec:prelim}

\neurips{\vspace{-0.3cm}}
\paragraph{Notation.} For a positive integer $n$, let $[n] := \{ 1, 2, \ldots, n\}$. For a finite set $\MS$, let $\distrs{\MS}$ denote the space of distributions on $\MS$. For $\MS = [n]$, we will write $\Delta^n := \Delta(\MS)$ and interpret elements of $\Delta^n$ as vectors in $\BR^n$. For a vector $v \in \BR^n$ and $j \in [n]$, we denote the $j$th coordinate of $v$ as $v(j)$. For vectors $v, w \in \BR^n$, write $\lng v, w \rng = \sum_{j=1}^n v(j)w(j)$.   The base-2 logarithm of $x > 0$ is denoted $\log x$.

\neurips{\vspace{-0.3cm}}
\paragraph{No-regret learning in games.} We consider a game $G$ with $m \in \BN$ players, where player $i \in [m]$ has \emph{action space} $\MA_i$ with $n_i := |\MA_i|$ actions. We may assume that $\MA_i = [n_i]$ for each player $i$. The \emph{joint action space} is $\MA := \MA_1 \times \cdots \times \MA_m$. The specification of the game $G$ is completed by a collection of \emph{loss functions} $\ML_1, \ldots, \ML_m : \MA \ra [0,1]$. For an action profile $a = (a_1, \ldots, a_m) \in \MA$ and $i \in [m]$, $\ML_i(a)$ is the loss player $i$ experiences when each player $i' \in [m]$ plays $a_{i'}$. A \emph{mixed strategy} $x_i \in \Delta(\MA_i)$ for player $i$ is a distribution over $\MA_i$, with the probability of playing action $j \in \MA_i$ given by $x_i(j)$. Given a mixed strategy profile $x = (x_1, \ldots, x_m)$ (or an action profile $a = (a_1, \ldots, a_m)$) and a player $i \in [m]$  we let $x_{-i}$ (or $a_{-i}$, respectively) denote the profile after removing the $i$th mixed strategy $x_i$ (or the $i$th action $a_i$, respectively).

The $m$ players play the game $G$ for a total of $T$ rounds. At the beginning of each round $t \in [T]$, each player $i$ chooses a mixed strategy $x_i\^t \in \distrs{\MA_i}$. The \emph{loss vector} of player $i$, denoted $\ell_i\^t \in [0,1]^{n_i}$, is defined as $\ell_i\^t(j) = \E_{a_{-i} \sim x_{-i}\^t} [\ML_i(j,a_{-i})]$. As a matter of convention, set $\ell_i\^0 = \bbzero$ to be the all-zeros vector. 
We consider the \emph{full-information setting} in this paper, meaning that player $i$ observes its full loss vector $\ell_i\^t$ for each round $t$. Finally, player $i$ experiences a loss of $\lng \ell_i\^t, x_i\^t \rng$. The goal of each player $i$ is to minimize its \emph{regret}, defined as:
$ 
\Reg{i}{T} := \sum_{t \in [T]} \lng x_i\^t, \ell_i\^t \rng - \min_{j \in [n_i]} \sum_{t \in [T]} \ell_i\^t(j).
$

\neurips{\vspace{-0.3cm}}
\paragraph{Optimistic hedge.} The \Opthedge algorithm chooses mixed strategies for player $i \in [m]$ as follows: at time $t = 1$, it sets $x_i\^1 = (1/n_i, \ldots, 1/n_i)$ to be the uniform distribution on $\MA_i$. Then for all $t < T$, player $i$'s strategy at iteration $t+1$ is defined as follows, for $j \in [n_i]$:
\begin{align}
  \label{eq:opt-hedge}
x_i\^{t+1}(j) := \frac{x_i\^t(j) \cdot \exp(-\eta \cdot (2\ell_i\^t(j) - \ell_i\^{t-1}(j)))}{\sum_{k \in [n_i]} x_i\^t(k) \cdot \exp(-\eta \cdot (2 \ell_i\^t(k) - \ell_i\^{t-1}(k)))}.
\end{align}
\Opthedge is a modification of Hedge, which performs the updates $x_i\^{t+1}(j) := \frac{x_i\^t(j) \cdot \exp(-\eta \cdot \ell_i\^t(j))}{\sum_{k \in [n_i]} x_i\^t(k) \cdot \exp(-\eta \cdot \ell_i\^t(k))}$.
The update (\ref{eq:opt-hedge}) modifies the Hedge update by replacing the loss vector $\ell_i\^t$ with a predictor of the \emph{following iteration's} loss vector, $\ell_i\^t + (\ell_i\^t - \ell_i\^{t-1})$.
Hedge corresponds to FTRL with a negative entropy regularizer (see, e.g., \citep{bubeck_convex_2015}), whereas \Opthedge corresponds to \emph{Optimistic} FTRL with a negative entropy regularizer  \citep{rakhlin_optimization_2013,rakhlin_online_2013}. 

\neurips{\vspace{-0.3cm}}
\paragraph{Distributions \& divergences.} For distributions $P,Q$ on a finite domain $[n]$, the \emph{KL divergence} between $P,Q$ is $\kld{P}{Q} = \sum_{j =1}^n P(j) \cdot \log \left( \frac{P(j)}{Q(j)} \right)$. The \emph{chi-squared divergence} between $P,Q$ is $\chisq{P}{Q} = \sum_{j=1}^n Q(j) \cdot \left( \frac{P(j)}{Q(j)} \right)^2 - 1 = \sum_{j=1}^n \frac{(P(j) - Q(j))^2}{Q(j)}$. For a distribution $P$ on $[n]$ and a vector $v \in \BR^n$, we write
$
\Var{P}{v} := \sum_{j=1}^n P(j) \cdot \left( v(j) - \sum_{k=1}^n P(k) v(k) \right)^2.
$
Also define $\norm{P}{v} := \sqrt{\sum_{j=1}^n P(j) \cdot v(j)^2}$. If further $P$ has full support, then define $\normst{P}{v} = \sqrt{\sum_{j=1}^n \frac{v(j)^2}{P(j)}}$. The above notations will often be used when $P$ is the mixed strategy profile $x_i$ for some player $i$ and $v$ is a loss vector $\ell_i$; in such a case the norms $\norm{P}{v}$ and $\normst{P}{v}$ are often called \emph{local norms}.

\section{Results}
\label{sec:results}

\neurips{\vspace{-0.3cm}}
Below we state our main theorem, which shows that when all players in a game play according to \Opthedge with appropriate step size, they all experience polylogarithmic individual regrets.
\begin{theorem}[Formal version of Theorem \ref{thm:main}]
  \label{thm:polylog-main}
  There are constants $C, C' >1$ so that the following holds. Suppose a time horizon $T \in \BN$ and a game $G$ with $m$ players and $n_i$ actions for each player $i\in [m]$ is given. Suppose all players play according to \Opthedge with any positive step size $\eta \leq \frac{1}{C \cdot m \log^4 T}$. Then for any $i \in [m]$, the regret of player $i$ satisfies
  \begin{align}
\Reg{i}{T} \leq \frac{\log n_i}{\eta} + C' \cdot \log T \label{eq:reg-eta-ub}.
  \end{align}  
  In particular, if the players' step size is chosen as $\eta = \frac{1}{C \cdot m \log^4 T}$, then the regret of player $i$ satisfies
  \begin{align}
\Reg{i}{T} \leq O \left( m \cdot \log n_i \cdot \log^4 T \right).\label{eq:reg-eta-fix}
  \end{align}
\end{theorem}
A common goal in the literature on learning in games is to obtain an algorithm that achieves fast rates whan played by all players, and so that each player $i$ still obtains the optimal rate of $O(\sqrt{T})$ in the adversarial setting (i.e., when $i$ receives an arbitrary sequence of losses $\ell_i\^1, \ldots, \ell_i\^T$). We show in Corollary \ref{cor:polylog-adv} (in the appendix) that by running \Opthedge with an adaptive step size, this is possible. Table \ref{tab:prior-work} compares our regret bounds discussed in this section to those of prior work. 

\section{Proof overview}
\label{sec:proof-overview}

\neurips{\vspace{-0.3cm}}
In this section we overview the proof of Theorem \ref{thm:polylog-main}; the full proof may be found in the appendix. 
\neurips{\vspace{-0.2cm}}
\subsection{New adversarial regret bound}

\neurips{\vspace{-0.2cm}}
\label{sec:adv-reg-bound}
The first step in the proof of Theorem \ref{thm:polylog-main} is to prove a new regret bound (Lemma \ref{lem:omwu-local} below) for \Opthedge that holds for an adversarial sequence of losses. We will show in later sections that when \emph{all} players play according to \Opthedge, the right-hand side of the regret bound (\ref{eq:adv-var-bound}) is bounded by a quantity that grows only poly-logarithmically in $T$. 
\begin{lemma}
  \label{lem:omwu-local}
  There is a constant $C > 0$ so that the following holds. 
Suppose any player $i \in [m]$ follows the \Opthedge updates (\ref{eq:opt-hedge}) with step size $\eta < 1/C$, for an arbitrary sequence of losses $\ell_i\^1, \ldots, \ell_i\^T \in [0,1]^{n_i}$. Then 
\begin{align}
  \label{eq:adv-var-bound}
  \hspace{-0.7cm}
  \Reg{i}{T} \leq \frac{\log n_i}{\eta} + \sum_{t=1}^T \left(\frac{\eta}{2} + C \eta^2\right)  \Var{x_i\^t}{\ell_i\^t - \ell_i\^{t-1}} - \sum_{t=1}^T \frac{ (1-C\eta)\eta}{2} \cdot \Var{x_i\^{t}}{\ell_i\^{t-1}}.
\end{align}
\end{lemma}
The detailed proof of Lemma~\ref{lem:omwu-local} can be found in Section~\ref{sec:adv-regbnd-proofs}, but we sketch the main steps here. The starting point is a refinement of~\cite[Lemma 3]{rakhlin_online_2013} (stated as Lemma \ref{lem:omwu-localnorm}), which gives an upper bound for $\Reg{i}{T}$ in terms of local norms corresponding to each of the iterates $x_i\^t$ of \Opthedge. The bound involves the difference between the \Opthedge iterates $x_i\^t$ and iterates $\tilde x_i\^t$ defined by 
  $\tilde x_i\^t = \frac{x_i\^{t}(j) \cdot \exp(-\eta \cdot (\ell_i\^t (j) - \ell_i\^{t-1}(j)))}{\sum_{k \in [n_i]} x_i\^{t}(k) \cdot \exp(-\eta \cdot (\ell_i\^t (k) - \ell_i\^{t-1}(k)))}$: 
  \begin{align}
     \hspace{-0.8cm}
    \Reg{i}{T} \leq & \frac{\log n_i}{\eta} + \sum_{t=1}^T \normst{x_i\^t}{ x_i\^t - \tilde x_i\^t } \sqrt{\Var{x_i\^t}{\ell_i\^t - \ell_i\^{t-1}}} 
                                 - \frac{1}{\eta} \sum_{t=1}^T \kld{\tilde x_i\^t}{x_i\^t} - \frac{1}{\eta} \sum_{t=1}^T \kld{x_i\^t}{ \tilde x_i\^{t-1}} .\label{eq:rakhlin-refinement-mainbody}
  \end{align}  
  We next show (in Lemma \ref{lem:kl-div-lb}) that $\kld{\til x_i\^t}{x_i\^t}$ and $\kld{x_i\^t}{\til x_i\^{t-1}}$ may be lower bounded by $(1/2 - O(\eta)) \cdot \chisq{\til x_i\^t}{x_i\^t}$ and $(1/2 - O(\eta)) \cdot \chisq{x_i\^t}{\til x_i\^{t-1}}$, respectively. Note it is a standard fact that the KL divergence between two distributions is upper bounded by the chi-squared distribution between them; by contrast, 
  Lemma \ref{lem:kl-div-lb} can exploit that $x_i\^t$, $\til x_i\^t$ and $\til x_i\^{t-1}$ are close to each other to show a reverse inequality. Finally, exploiting the exponential weights-style functional relationship between $x_i\^t$ and $\til x_i\^{t-1}$, we show (in Lemma \ref{lem:chi2-variance}) that the $\chi^2$-divergence $\chisq{x_i\^t}{\til x_i\^{t-1}}$ may be lower bounded by $(1 - O(\eta)) \cdot \eta^2 \cdot \Var{x_i\^t}{\ell_i\^{t-1}}$, leading to the term $\frac{(1 - C\eta) \eta}{2} \Var{x_i\^t}{\ell_i\^{t-1}}$ being subtracted in (\ref{eq:adv-var-bound}). The $\chi^2$-divergence $\chisq{\til x_i\^t}{ x_i\^{t}}$, 
  as well as the term $\normst{x_i\^t}{x_i\^t - \til x_i\^t}$ in (\ref{eq:rakhlin-refinement-mainbody}) are bounded in a similar manner to obtain (\ref{eq:adv-var-bound}).

  \subsection{Finite differences}
  \label{sec:fds}

\neurips{  \vspace{-0.2cm}}
  Given Lemma \ref{lem:omwu-local}, in order to establish Theorem \ref{thm:polylog-main}, it suffices to show Lemma \ref{lem:bound-l-dl} below. Indeed, (\ref{eq:var-l-dl-main}) 
  below implies that the right-hand side of (\ref{eq:adv-var-bound}) is bounded above by $\frac{\log n_i}{\eta} + \eta \cdot O(\log^5 T)$, which is bounded above by 
  $O(m \log n_i \log^4 T)$ for the choice $\eta = \Theta \left( \frac{1}{m \cdot \log^4 T} \right)$ of Theorem \ref{thm:polylog-main}.\footnote{Notice that the factor $\frac 12$ in (\ref{eq:var-l-dl-main}) is not important for this argument -- any constant less than 1 would suffice.} 

\begin{lemma}[Abbreviated; detailed version in Section \ref{sec:polylog-main-completing-proof}]
  \label{lem:bound-l-dl}
  Suppose all players play according to \Opthedge with step size $\eta$ satifying $1/T \leq \eta \leq \frac{1}{C m \cdot \log^4 T}$ for a sufficiently large constant $C$. Then for any $i \in [m]$, the losses $\ell_i\^1, \ldots, \ell_i\^T \in \BR^{n_i}$ for player $i$ satisfy:
\begin{align}
  \label{eq:var-l-dl-main}
\sum_{t=1}^T \Var{x_i\^t}{\ell_i\^t - \ell_i\^{t-1}} \leq \frac{1}{2} \cdot \sum_{t=1}^T \Var{x_i\^t}{\ell_i\^{t-1}} + O \left( \log^5 T\right).
  \end{align}
\end{lemma}

The  definition below allows us to streamline our notation when proving Lemma \ref{lem:bound-l-dl}.
\begin{defn}[Finite differences]
  \label{def:fd}
  Suppose $L = (L\^1, \ldots, L\^T)$ is a sequence of vectors $L\^t\in \BR^n$. For integers $h \geq 0$, the \emph{order-$h$ finite difference sequence} for the sequence $L$, denoted by $\fd{h}{L}$, is the sequence $\fd{h}{L} := (\fds{h}{L}{1}, \ldots, \fds{h}{L}{T-h})$ defined recursively as: $\fds{0}{L}{t} := L\^t$ for all $1 \leq t \leq T$, and
  \begin{equation}
    \label{eq:fd-defn}
\fds{h}{L}{t} := \fds{h-1}{L}{t+1} - \fds{h-1}{L}{t}
\end{equation}
for all $h \geq 1$, $1 \leq t \leq T-h$.\footnote{We remark that while Definition \ref{def:fd} is stated for a 1-indexed sequence $L\^1, L\^2, \ldots$, we will also occasionally consider 0-indexed sequences $L\^0, L\^1, \ldots$, in which case the same recursive definition (\ref{eq:fd-defn}) holds for the finite differences $\fds{h}{L}{t}$, $t \geq 0$.} 
\end{defn}

\begin{remark}
  \label{rem:fds-alt}
Notice that another way of writing (\ref{eq:fd-defn}) is:
$ \fd{h}{L} = \fd{1}{\fd{h-1}{L}}.$
We also remark for later use that
$ \fds{h}{L}{t} = \sum_{s=0}^h {h \choose s} (-1)^{h-s} L\^{t+s}.$
\end{remark}
Let $H = \log T$, where $T$ denotes the fixed time horizon from Theorem \ref{thm:polylog-main} (and thus Lemma \ref{lem:bound-l-dl}). In the proof of Lemma \ref{lem:bound-l-dl}, we will bound the finite differences of order $h \leq H$ for certain sequences. 
The bound (\ref{eq:var-l-dl-main}) of Lemma \ref{lem:bound-l-dl} may be rephased as upper bounding $\sum_{t=1}^T \Var{x_i\^t}{\fds{1}{\ell_i}{t-1}}$, by $\frac 12 \sum_{t=1}^T \Var{x_i}{\ell_i\^{t-1}}$; to prove this, we proceed in two steps:
\begin{enumerate}[leftmargin=12pt]
\item \label{it:upwards-ind} (\emph{Upwards induction step}) First, in Lemma \ref{lem:dh-bound} below, we find an upper bound on $\left\| \fds{h}{\ell_i}{t} \right\|_\infty$ for all $t \in [T]$, $h \geq 0$, which decays exponentially in $h$ for $h \leq H$. This is done via \emph{upwards induction} on $h$, i.e., first proving the base case $h = 0$ using boundedness of the losses $\ell_i\^t$ and then $h = 1, 2, \ldots$ inductively. The main technical tool we develop for the inductive step is a weak form of the chain rule for finite differences, Lemma \ref{lem:fd-analytic}. The inductive step uses the fact that all players are following \Opthedge to relate the $h$th order finite differences of player $i$'s loss sequence $\ell_i\^t$ to the $h$th order finite differences of the strategy sequences $x_{i'}\^t$ for players $i' \neq i$; then we use the exponential-weights style updates of \Opthedge and Lemma \ref{lem:fd-analytic} to relate the $h$th order finite differences of the strategies $x_{i'}\^t$ to the $(h-1)$th order finite differences of the losses $\ell_{i'}\^t$.
\item \label{it:downwards-ind} (\emph{Downwards induction step}) We next show that for all $0 \leq h \leq  H$, $\sum_{t=1}^T \Var{x_i\^t}{\fds{h+1}{\ell_i}{t-1}}$ is bounded above by $c_h \cdot \sum_{t=1}^T\Var{x_i\^t}{\fds{h}{\ell_i}{t-1}} + \mu_h$, for some $c_h < 1/2$ and $\mu_h < O(\log^5 T)$.
  This shown via \emph{downwards induction} on $h$, namely first establishing the base case $h = H$ by using the result of item \ref{it:upwards-ind} for $h = H$ and then treating the cases $h = H-1, H-2, \ldots, 0$. The inductive step makes use of the discrete Fourier transform (DFT) to relate the finite differences of different orders (see Lemmas \ref{lem:d210} and \ref{lem:freq-cauchy}). In particular, Parseval's equality together with a standard relationship between the DFT of the finite differences of a sequence to the DFT of that sequence allow us to first prove the inductive step in the frequency domain and then transport it back to the original (time) domain.
\end{enumerate}
In the following subsections we explain in further detail how the two steps above are completed.

\subsection{Upwards induction proof overview}
\label{sec:upwards-ind}

\neurips{\vspace{-0.2cm}}
Addressing item \ref{it:upwards-ind} in the previous subsection, the  lemma below gives a bound on the supremum norm of the $h$-th order finite differences of each player's loss vector, when all players play according to \Opthedge and experience losses according to their loss functions $\ML_1, \ldots, \ML_m : \MA \ra [0,1]$. 

\begin{lemma}[Abbreviated]
  \label{lem:dh-bound}
  Fix a step size $\eta > 0$ satisfying $\eta \leq o \left( \frac{1}{m\log T} \right)$.
  If all players follow \Opthedge updates with step size $\eta$, then for any player $i \in [m]$, integer $h$ satisfying $0 \leq h \leq H$, and time step $t \in [T-h]$, it holds that
$
\| \fds{h}{\ell_i}{t} \|_\infty \leq O( m \eta) ^h \cdot h^{O(h)}.
$ 
\end{lemma}
A detailed version of Lemma \ref{lem:dh-bound}, together with its full proof, may be found in Section \ref{sec:dh-bound-proof}. We next give a proof overview of Lemma \ref{lem:dh-bound} for the case of 2 players, i.e., $m = 2$; we show in Section \ref{sec:dh-bound-proof} how to generalize this computation to general $m$. Below we introduce the main technical tool in the proof, a ``boundedness chain rule,'' and then outline how it is used to prove Lemma \ref{lem:dh-bound}.

\neurips{\vspace{-0.3cm}}
\paragraph{Main technical tool for Lemma \ref{lem:dh-bound}: boundedness chain rule.}
We say that a function $\phi : \BR^n \ra \BR$ is a \emph{softmax-type} function  if there are real numbers $\xi_1, \ldots, \xi_n$ and some $j \in [n]$ so that for all $(z_1, \ldots, z_n) \in \BR^n$,
$ 
\phi((z_1, \ldots, z_n)) = \frac{\exp(z_j)}{\sum_{k =1}^n \xi_k \cdot \exp(z_k)}.
$ 
Lemma \ref{lem:fd-analytic} below may be interpreted as a ``boundedness chain rule'' for finite differences. To explain the context for this lemma, recall that given an infinitely differentiable vector-valued function $L : \BR \ra \BR^n$ and an infinitely differentiable function $\phi : \BR^n \ra \BR$, the higher order derivatives of the function $\phi(L(t))$ may be computed in terms of those of $L$ and $\phi$ using the chain rule. Lemma \ref{lem:fd-analytic} considers an analogous setting where the input variable $t$ to $L$ is discrete-valued, taking values in $[T]$ (and so we identify the function $L$ with the sequence $L\^1, \ldots, L\^T$). In this case, the \emph{higher order finite differences} of the sequence $L\^1, \ldots, L\^T$ (Definition \ref{def:fd}) take the place of the higher order derivatives of $L$ with respect to $t$. Though there is no generic chain rule for finite differences, Lemma \ref{lem:fd-analytic} states that, at least when $\phi$ is a softmax-type function, 
we may \emph{bound} the higher order finite differences of the sequence $\phi(L\^1), \ldots, \phi(L\^T)$. In the lemma's statement we let $\phi \circ L$ denote the sequence $\phi(L\^1), \ldots, \phi(L\^T)$. 
\begin{lemma}[``Boundedness chain rule'' for finite differences; abbreviated]
  \label{lem:fd-analytic}
  Suppose that $h,n \in \BN$,  $\phi : \BR^n \ra \BR$ is a softmax-type function, and $L = (L\^1, \ldots, L\^T)$ 
  is a sequence of vectors in $\BR^n$ satisfying $\| L\^t \|_\infty \leq 1$ for $t \in [T]$. 
 Suppose for some $\alpha \in (0,1)$, for each $0 \leq h' \leq h$ and $t \in [T-h']$, it holds that $\| \fd{h'}{L}\^t \|_\infty \leq O( {\alpha^{h'}} ) \cdot (h')^{O(h')}$. 
 Then for all $t \in [T-h]$, 
  \begin{align*}
| \fds{h}{(\phi \circ L)}{t}|\leq O( \alpha^h) \cdot h^{O(h)}. 
  \end{align*}
\end{lemma}
A detailed version of Lemma \ref{lem:fd-analytic} may be found in Section \ref{sec:bcr-proof-2}. While Lemma \ref{lem:fd-analytic} requires $\phi$ to be a softmax-type function for simplicity (and this is the only type of function $\phi$ we will need to consider for the case $m = 2$) 
we remark that the detailed version of Lemma \ref{lem:fd-analytic} allows $\phi$ to be from a more general family of analytic functions whose higher order derivatives are appropriately bounded. The proof of Lemma \ref{lem:dh-bound} for all $m \geq 2$ requires that more general form of Lemma \ref{lem:fd-analytic}.

The proof of Lemma \ref{lem:fd-analytic} proceeds by considering the Taylor expansion $P_\phi(\cdot)$ of the function $\phi$ at the origin, which we write as follows: for $z = (z_1, \ldots, z_n) \in \BR^n$, $P_\phi(z) := \sum_{k \geq 0, \gamma \in \BZ_{\geq 0}^n : \ |\gamma| = k} a_\gamma z^\gamma$, where $a_\gamma \in \BR$, $|\gamma|$ denotes the quantity $\gamma_1 + \cdots + \gamma_n$ and $z^\gamma$ denotes $z_1^{\gamma_1} \cdots z_n^{\gamma_n}$. The fact that $\phi$ is a softmax-type function ensures that the radius of convergence of its Taylor series is at least 1, i.e., $\phi(z) = P_\phi(z)$ for any $z$ satisfying $\| z \|_\infty \leq 1$. By the assumption that $\| L\^t \|_\infty \leq 1$ for each $t$, we may therefore decompose $\fds{h}{(\phi \circ L)}{t}$ as:
\begin{align}
  \fds{h}{(\phi \circ L)}{t} = \sum_{k \geq 0, \gamma \in \BZ_{\geq 0}^n: \ |\gamma| = k} a_\gamma \cdot \fds{h}{L^\gamma}{t},\label{eq:phi-informal-decompose}
\end{align}
where $L^\gamma$ denotes the sequence of scalars $(L^\gamma)\^t := (L\^t)^\gamma$ for all $t$. 
The fact that $\phi$ is a softmax-type function allows us to establish strong bounds on $|a_\gamma|$ for each $\gamma$ in Lemma \ref{lem:softmax-ak-bound}. 
The proof of Lemma \ref{lem:softmax-ak-bound} bounds the $|a_\gamma|$ by exploiting the simple form of the derivative of a softmax-type function to decompose each $a_\gamma$ into a sum of $|\gamma|!$ terms. Then we establish a bijection between the terms of this decomposition and graph structures we refer to as \emph{factorial trees}; that bijection together with the use of an appropriate generating function allow us to complete the proof of Lemma \ref{lem:softmax-ak-bound}.

Thus, to prove Lemma \ref{lem:fd-analytic}, it suffices to bound $\left| \fds{h}{L^\gamma}{t} \right|$ for all $\gamma$. We do so by using Lemma \ref{lem:expand-pow-seq}. 
\begin{lemma}[Abbreviated; detailed vesion in Section \ref{sec:bcr-proof}]
  \label{lem:expand-pow-seq}
  Fix any $h \geq 0$, a multi-index $\gamma \in \BZ_{\geq 0}^n$ and set $k = |\gamma|$. For each of the $k^h$ functions $\pi : [h] \ra [k]$, and for each $r \in [k]$, there are integers $h'_{\pi,r} \in \{0, 1, \ldots, h\}$, $t'_{\pi,r} \geq 0$, and $j'_{\pi,r} \in [n]$, so that the following holds. For any sequence $L\^1, \ldots, L\^T \in \BR^n$ of vectors, it holds that, for each $t \in [T-h]$, 
  \begin{align}
\fds{h}{L^\gamma}{t} = \sum_{\pi : [h] \ra [k]} \prod_{r=1}^k \fds{h'_{\pi,r}}{(L(j'_{\pi,r}))}{t+t'_{\pi,r}}\label{eq:prod-expand-informal}.
  \end{align}
\end{lemma}
Lemma \ref{lem:expand-pow-seq} expresses the $h$th order finite differences of the sequence $L^\gamma$ as a sum of $k^h$ terms, each of which is a product of $k$ finite order differences of a sequence $L\^t(j'_{\pi,r})$ (i.e., the $j'_{\pi,r}$th coordinate of the vectors $L\^t$). Crucially, when using Lemma \ref{lem:expand-pow-seq} to  prove Lemma \ref{lem:fd-analytic}, the assumption of Lemma \ref{lem:fd-analytic} gives that for each $j' \in [n]$, each $h' \in [h]$, and each $t' \in [T-h']$, we have the bound $\left| \fds{h'}{L(j')}{t'} \right| \leq O ( {\alpha^{h'}}) \cdot (h')^{O(h')}$. 
These assumed bounds may be used to bound the right-hand side of (\ref{eq:prod-expand-informal}), which together with Lemma \ref{lem:expand-pow-seq} and (\ref{eq:phi-informal-decompose}) lets us complete the proof of Lemma \ref{lem:fd-analytic}. 

\neurips{\vspace{-0.3cm}}
\paragraph{Proving Lemma \ref{lem:dh-bound} using the boundedness chain rule.}
Next we discuss how Lemma \ref{lem:fd-analytic} is used to prove Lemma \ref{lem:dh-bound}, namely to bound $\| \fds{h}{\ell_i}{t} \|_\infty$ for each $t \in [T-h]$, $i \in [m]$, and $0 \leq h \leq H$. Lemma \ref{lem:dh-bound} is proved using induction, with the base case $h = 0$ being a straightforward consequence of the fact that $\| \fds{0}{\ell_i}{t} \|_\infty = \| \ell_i\^t \|_\infty \leq 1$ for all $i \in [m], t \in [T]$. For the rest of this section we focus on the inductive case, i.e., we pick some $h \in [H]$ and assume Lemma \ref{lem:dh-bound} holds for all $h' < h$. 

The first step is to reduce the claim 
of Lemma \ref{lem:dh-bound} to the claim that the upper bound $\|\fds{h}{x_i}{t} \|_1 \leq O \left( m \eta \right)^h \cdot h^{O(h)}$ holds for each $t \in [T-h], i \in [m]$. Recalling that we are only sketching here the case  $m=2$ for  simplicity, this reduction proceeds as follows: for $i \in \{1,2\}$, define the matrix $A_i \in \BR^{n_1 \times n_2}$ by $(A_i)_{a_1a_2}= \ML_i(a_1, a_2)$, for $a_1 \in [n_1], a_2 \in [n_2]$. We have assumed that all players are using \Opthedge and thus $\ell_i\^t = \E_{a_{i'} \sim x_{i'}\^t,\ \forall i' \neq i} [\ML_i(a_1, \ldots, a_n)]$; for our case here ($m=2$), this may be rewritten as $\ell_1\^t = A_1 x_2\^t$, $\ell_2\^t = A_2^\t x_1\^t$. Thus
\begin{align}
\| \fds{h}{\ell_1}{t} \|_\infty = \left\|A_1 \cdot \sum_{s=0}^h {h \choose s} (-1)^{h-s}  x_2\^{t+s} \right\|_\infty \leq \left\| \sum_{s=0}^h {h \choose s} (-1)^{h-s}  x_2\^{t+s} \right\|_1 = \| \fds{h}{x_2}{t} \|_1\nonumber,
\end{align}
where the first equality is from Remark \ref{rem:fds-alt} and the inequality follows since all entries of $A_1$ have absolute value $\le 1$. A similar computation allows us to show $\| \fds{h}{\ell_2}{t} \|_\infty \leq \| \fds{h}{x_1}{t} \|_1$.

To complete the inductive step it remains to upper bound the quantities $\| \fds{h}{x_i}{t} \|_1$ for $i \in [m], t \in [T-h]$. To do so, we note that the definition of the \Opthedge updates (\ref{eq:opt-hedge}) implies that for any $i \in [m], t \in [T], j \in [n_i]$, and $t' \geq 1$, we have 
\begin{align}
  \label{eq:ttp-extrapolate}
x_i\^{t+t'}(j) = \frac{x_i\^{t}(j) \cdot \exp\left( \eta \cdot \left( \ell_i\^{t-1}(j) - \sum_{s=0}^{t'-1} \ell_i\^{t+s}(j) - \ell_i\^{t+t'-1}(j)\right) \right)}{\sum_{k=1}^{n_i} x_i\^{t}(k) \cdot \exp\left( \eta \cdot \left( \ell_i\^{t-1}(k) - \sum_{s=0}^{t'-1} \ell_i\^{t+s}(k) - \ell_i\^{t+t'-1}(k)\right) \right)}.
\end{align}
For $t \in [T]$, $t' \geq 0$, set
$
\bar \ell_{i,t}\^{t'} := \eta \cdot \left( \ell_i\^{t-1} - \sum_{s=0}^{t'-1} \ell_i\^{t+s} - \ell_i\^{t+t'-1}\right).
$ 
Also, for each $i \in [m]$, $j \in [n_i]$, $t \in [T]$, and any vector $z = (z(1), \ldots, z(n_i)) \in \BR^{n_i}$ define $
\phi_{t,i,j}(z) := \frac{ x_i\^t(j) \cdot \exp\left(  z(j)\right)}{\sum_{k=1}^{n_i} x_i\^{t}(k) \cdot \exp\left( z(k)\right)}.$ 
Thus (\ref{eq:ttp-extrapolate}) gives that for $t' \geq 1$, $x_i\^{t+t'}(j) = \phi_{t,i,j}(\bar \ell_{i,t}\^{t'})$. Viewing $t$ as a fixed parameter and letting $t'$ vary, it follows that for $h \geq 0$ and $t' \geq 1$, $\fds{h}{x_i\^{t+\cdot}(j)}{t'} = \fds{h}{(\phi_{t,i,j} \circ \bar \ell_{i,t})}{t'}$. 

Recalling that our goal is to bound $|\fds{h}{x_i(j)}{t+1}|$ for each $t$, 
we can do so by using Lemma \ref{lem:fd-analytic} with $\phi = \phi_{t,i,j}$ and $\alpha = O(m \eta)$, if we can show that its precondition is met, i.e.~that $\| \fds{h'}{\bar \ell_{i,t}}{t'} \|_\infty \leq \frac{1}{B_1} \cdot \alpha^{h'} \cdot (h')^{B_0 h'}$ for all $h' \leq h$, the appropriate $\alpha$ and appropriate constants $B_0, B_1$. Helpfully, the definition of $\bar \ell_{i,t}\^{t'}$ as a partial sum allows us to relate the $h'$-th order finite differences of the sequence $\bar \ell_{i,t}\^{t'}$ to the $(h'-1)$-th order finite differences of the sequence $\ell_i\^t$ as follows:
\begin{equation}
  \label{eq:barl-l}
\fds{h'}{\bar \ell_{i,t}}{t'} = \eta \cdot \fds{h'-1}{\ell_i}{t+t'-1} -2\eta \cdot \fds{h'-1}{\ell_i}{t+t'}.
\end{equation}
Since $h'-1 < h$ for $h' \leq h$, the inductive assumption of Lemma \ref{lem:dh-bound} gives a bound on the $\ell_\infty$-norm of the terms on the right-hand side of (\ref{eq:barl-l}), which are sufficient for us to apply Lemma \ref{lem:fd-analytic}. Note that the inductive assumption gives an upper bound on $\| \fds{h'-1}{\ell_i}{t} \|_\infty$ that only scales with $\alpha^{h'-1}$, 
whereas Lemma \ref{lem:fd-analytic} requires scaling of $\alpha^{h'}$. This discrepancy is corrected by the factor of $\eta$ on the right-hand side of (\ref{eq:barl-l}), which gives the desired scaling $\alpha^{h'}$ (since $\eta < \alpha$ for the choice $\alpha = O(m\eta)$).

\subsection{Downwards induction proof overview}
\label{sec:downwards-ind}

\neurips{\vspace{-0.3cm}}
In this section we discuss in further detail item \ref{it:downwards-ind} in Section \ref{sec:fds}; in particular, we will show that there is a parameter $\mu = \tilde \Theta(\eta m)$ so that for all integers $h$ satisfying $H-1 \geq h \geq 0$,
  \begin{equation}
    \label{eq:fds-H-ind-informal}
\sum_{t=1}^{T-h-1} \Var{x_i\^t}{\fds{h+1}{\ell_i}{t}} \leq O(1/H) \cdot \sum_{t=1}^{T-h} \Var{x_i\^t}{\fds{h}{\ell_i}{t}} +  \tilde O \left( \mu^{2h} \right), 
\end{equation}
where $\tilde O$ hides factors polynomial in $\log T$. 
The validity of (\ref{eq:fds-H-ind-informal}) for $h = 0$ implies Lemma \ref{lem:bound-l-dl}. On the other hand, as long we choose the value $\mu$ in (\ref{eq:fds-H-ind-informal}) to satisfy $\mu \geq m \eta H^{\Omega(1)}$, 
then Lemma \ref{lem:dh-bound} implies that $\sum_{t=1}^{T-H} \Var{x_i\^t}{\fds{H}{\ell_i}{t}} \leq O(\mu^{2H})$. This gives that (\ref{eq:fds-H-ind-informal}) holds for $h = H-1$. To show that (\ref{eq:fds-H-ind-informal}) holds for all $H-1 > h \geq 0$, we use downwards induction; fix any $h$, and assume that (\ref{eq:fds-H-ind-informal}) has been shown for all $h'$ satisfying $h < h' \leq H-1$. Our main tool in the inductive step is to apply Lemma \ref{lem:d210} below. To state it, 
for $\zeta > 0,\ n \in \BN$, we say that a sequence of distributions $P\^1, \ldots, P\^T \in \Delta^n$ is \emph{$\zeta$-consecutively close} if for each $1 \leq t < T$, it holds that $\max \left\{\left\| \frac{P\^t}{P\^{t+1}}\right\|_\infty, \left\| \frac{P\^{t+1}}{P\^t}\right\|_\infty\right\} \leq 1 + \zeta$.\footnote{Here, for distributions $P,Q \in \Delta^n$, $\frac{P}{Q} \in \BR^n$ denotes the vector whose $j$th entry is $P(j) / Q(j)$.} Lemma \ref{lem:d210} shows that given a sequence of vectors for which the variances of its second-order finite differences are bounded by the variances of its first-order finite differences, a similar relationship holds between its first- and zeroth-order finite differences.
\begin{lemma}
  \label{lem:d210}
  There is a sufficiently large constant $C_0 > 1$ so that the following holds. For any $M, \zeta, \alpha > 0$ and $n \in \BN$, suppose that $P\^1, \ldots, P\^T \in \Delta^n$ and $Z\^1, \ldots, Z\^T \in [-M,M]^n$ satisfy the following conditions:
  \begin{enumerate}
  \item \label{it:prec-consec-close-body} The sequence $P\^1, \ldots, P\^T$ is $\zeta$-consecutively close for some $\zeta \in [1/(2T), \alpha^4/C_0]$. 
  \item \label{it:prec-vars-body} It holds that 
$ 
      \sum_{t=1}^{T-2} \Var{P\^t}{\fds{2}{Z}{t}} \leq \alpha \cdot \sum_{t=1}^{T-1} \Var{P\^t}{\fds{1}{Z}{t}} + \mu.
$ 
\end{enumerate}
Then\ \ 
$ 
\sum_{t=1}^{T-1} \Var{P\^t}{\fds{1}{Z}{t}} \leq \alpha \cdot (1+ \alpha) \sum_{t=1}^T \Var{P\^t}{Z\^t} + \frac{\mu}{\alpha} + \frac{C_0 M^2}{\alpha^3}.
$

\end{lemma}
Given Lemma \ref{lem:d210}, the inductive step for establishing (\ref{eq:fds-H-ind-informal}) is straightforward: we apply Lemma \ref{lem:d210} with $P\^t = x_i\^t$ and $Z\^t = \fds{h}{\ell_i}{t}$ for all $t$. The fact that $x_i\^t$ are updated with \Opthedge may be used to establish that precondition \ref{it:prec-consec-close-body} of Lemma \ref{lem:d210} holds. Since $\fds{1}{Z}{t} = \fds{h+1}{\ell_i}{t}$ and $\fds{2}{Z}{t} = \fds{h+2}{\ell_i}{t}$, that the inductive hypothesis (\ref{eq:fds-H-ind-informal}) holds for $h+1$ implies that precondition \ref{it:prec-vars-body} of Lemma \ref{lem:d210} holds for appropriate $\alpha, \mu > 0$. Thus Lemma \ref{lem:d210} implies that (\ref{eq:fds-H-ind-informal}) holds for the value $h$, which completes the inductive step.

\neurips{\vspace{-0.3cm}}
\paragraph{On the proof of Lemma \ref{lem:d210}.}
Finally we discuss the proof of Lemma \ref{lem:d210}. One technical challenge is the fact that the vectors $P\^t$ are not constant functions of $t$, but rather change slowly (as constrained by being $\zeta$-consecutively close). The main tool for dealing with this difficulty is Lemma \ref{lem:covvar-close}, which shows that for a $\zeta$-consecutively close sequence $P\^t$, for any vector $Z\^t$, $\frac{\Var{P\^t}{Z\^t}}{\Var{P\^{t+1}}{Z\^t}} \in [1-\zeta, 1+\zeta]$. This fact, together with some algebraic manipulations, lets us to reduce to the case that all $P\^t$ are equal. It is also relatively straightforward to reduce to the case that $\lng P\^t, Z\^t \rng = 0$ for all $t$, i.e., so that $\Var{P\^t}{Z\^t} = \norm{P\^t}{Z\^t}^2$. We may further separate $\norm{P\^t}{Z\^t}^2 = \sum_{j=1}^n P\^t(j) \cdot (Z\^t(j))^2$ into its individual components $P\^t(j) \cdot (Z\^t(j))^2$, and treat each one separately, thus allowing us to reduce to a one-dimensional problem. Finally, we make one further reduction, which is to replace the finite differences $\fd{h}{(\cdot)}$ in Lemma \ref{lem:d210} with \emph{circular finite differences}, defined below:
\begin{defn}[Circular finite difference]
  Suppose $L = (L\^0, \ldots, L\^{S-1})$ is a sequence of vectors $L\^t \in \BR^n$. For integers $h \geq 0$, the \emph{level-$h$ circular finite difference sequence} for the sequence $L$, denoted by $\fdc{h}{L}$, is the sequence defined recursively as: $\fdcs{0}{L}{t} = L\^t$ for all $0 \leq t < S$, and
  \begin{equation}
    \fdcs{h}{L}{t} = \begin{cases}
      \fdcs{h-1}{L}{t+1} - \fdcs{h-1}{L}{t} \quad : 0 \leq t \leq S-2 \\
      \fdcs{h-1}{L}{1} - \fdcs{h-1}{L}{T} \quad : t = S-1.
    \end{cases}
  \end{equation}
\end{defn}
Circular finite differences for a sequence $L\^0, \ldots, L\^{S-1}$ are defined similarly to finite differences (Definition \ref{def:fd}) except that unlike for finite differences, where $\fds{h}{L}{S-h}, \ldots, \fds{h}{L}{S-1}$ are not defined, $\fdcs{h}{L}{S-h}, \ldots, \fdcs{h}{L}{S-1}$ are defined by ``wrapping around'' back to the beginning of the sequence. The above-described reductions, which are worked out in detail in Section \ref{sec:d210-proof}, allow us to reduce proving Lemma \ref{lem:d210} to proving the following simpler lemma:
\begin{lemma}
  \label{lem:freq-cauchy}
  Suppose $\mu \in \BR$, $\alpha > 0$, and $W\^0, \ldots, W\^{S-1} \in \BR$ is a sequence of reals satisfying
  \begin{equation}
    \label{eq:d2d1}
\sum_{t=0}^{S-1} \left( \fdcs{2}{W}{t} \right)^2 \leq \alpha \cdot \sum_{t=0}^{S-1} \left( \fdcs{1}{W}{t} \right)^2 + \mu.
\end{equation}
Then\ \ 
$
\sum_{t=0}^{S-1} \left( \fdcs{1}{W}{t} \right)^2 \leq \alpha \cdot \sum_{t=1}^{S-1} (W\^t)^2 + \mu/\alpha.
$
\end{lemma}
To prove Lemma \ref{lem:freq-cauchy}, we apply the discrete Fourier transform to both sides of (\ref{eq:d2d1}) and use the Cauchy-Schwarz inequality in frequency domain. For a sequence $W\^0, \ldots, W\^{S-1} \in \BR$, its (discrete) \emph{Fourier transform} is the sequence $\wh{W}\^0, \ldots, \wh{W}\^{S-1}$ defined by $\wh{W}\^s = \sum_{t=0}^{S-1} W\^t \cdot e^{-\frac{2\pi i st}{S}}$. 
Below we prove Lemma \ref{lem:freq-cauchy} for the special case $\mu = 0$; we defer the general case to Section \ref{sec:d210-prelim}.
\begin{proof}[Proof of Lemma \ref{lem:freq-cauchy} for special case $\mu = 0$]
  We have the following:
  \begin{align}
  \hspace{-0.6cm}    S \cdot \sum_{t=1}^T \left( \fdcs{1}{W}{t} \right)^2  = \sum_{s=1}^T \left| \widehat{\fdc{1}{W}}\^s \right|^2
    = \sum_{s=1}^T \left| \widehat{W}\^s  (e^{2\pi i s /T - 1}) \right|^2 \leq
    \sqrt{\sum_{s=1}^T \left| \widehat{W}\^s \right|^2}  \sqrt{\sum_{s=1}^T \left| \widehat{W}\^s \right|^2  \left| e^{2\pi i s/T} - 1\right|^4}\nonumber,
  \end{align}
  where the first equality uses Parseval's equality, the second uses Fact \ref{fac:fourier-circ} (in the appendix) for $h=1$, and the inequality uses Cauchy-Schwarz. By Parseval's inequality and Fact \ref{fac:fourier-circ} for $h=2$, the right-hand side of the above equals $S \cdot \sqrt{ \sum_{t=1}^T (W\^t)^2} \cdot \sqrt{\sum_{t=1}^T \left( \fdcs{2}{W}{t}\right)^2}$, which, by assumption, is at most $S \cdot \sqrt{ \sum_{t=1}^T (W\^t)^2} \cdot \sqrt{\alpha \cdot \sum_{t=1}^T \left( \fdcs{1}{W}{t} \right)^2 }$. Rearranging terms completes the proof. 
\end{proof}


\neurips{
\begin{ack}
Use unnumbered first level headings for the acknowledgments. All acknowledgments
go at the end of the paper before the list of references. Moreover, you are required to declare
funding (financial activities supporting the submitted work) and competing interests (related financial activities outside the submitted work).
More information about this disclosure can be found at: \url{https://neurips.cc/Conferences/2021/PaperInformation/FundingDisclosure}.

Do {\bf not} include this section in the anonymized submission, only in the final paper. You can use the \texttt{ack} environment provided in the style file to autmoatically hide this section in the anonymized submission.
\end{ack}
}


{
\neurips{  \small}
  \bibliographystyle{alpha}
  \bibliography{multi-agent}
}

\newpage
\neurips{
\section*{Checklist}

\begin{enumerate}

\item For all authors...
\begin{enumerate}
  \item Do the main claims made in the abstract and introduction accurately reflect the paper's contributions and scope?
    \answerYes{}
  \item Did you describe the limitations of your work?
    \answerYes{}
  \item Did you discuss any potential negative societal impacts of your work?
    \answerNA{}
  \item Have you read the ethics review guidelines and ensured that your paper conforms to them?
    \answerYes{}
\end{enumerate}

\item If you are including theoretical results...
\begin{enumerate}
  \item Did you state the full set of assumptions of all theoretical results?
    \answerYes{See Sections \ref{sec:prelim} and \ref{sec:results}.}
	\item Did you include complete proofs of all theoretical results?
    \answerYes{See Section \ref{sec:proof-overview} and the appendix (in the supplementary material).}
\end{enumerate}

\item If you ran experiments...
\begin{enumerate}
  \item Did you include the code, data, and instructions needed to reproduce the main experimental results (either in the supplemental material or as a URL)?
    \answerNA{}
  \item Did you specify all the training details (e.g., data splits, hyperparameters, how they were chosen)?
    \answerNA{}
	\item Did you report error bars (e.g., with respect to the random seed after running experiments multiple times)?
    \answerNA{}
	\item Did you include the total amount of compute and the type of resources used (e.g., type of GPUs, internal cluster, or cloud provider)?
    \answerNA{}
\end{enumerate}

\item If you are using existing assets (e.g., code, data, models) or curating/releasing new assets...
\begin{enumerate}
  \item If your work uses existing assets, did you cite the creators?
    \answerNA{}
  \item Did you mention the license of the assets?
    \answerNA{}
  \item Did you include any new assets either in the supplemental material or as a URL?
    \answerNA{}
  \item Did you discuss whether and how consent was obtained from people whose data you're using/curating?
    \answerNA{}
  \item Did you discuss whether the data you are using/curating contains personally identifiable information or offensive content?
    \answerNA{}
\end{enumerate}

\item If you used crowdsourcing or conducted research with human subjects...
\begin{enumerate}
  \item Did you include the full text of instructions given to participants and screenshots, if applicable?
    \answerNA{}
  \item Did you describe any potential participant risks, with links to Institutional Review Board (IRB) approvals, if applicable?
    \answerNA{}
  \item Did you include the estimated hourly wage paid to participants and the total amount spent on participant compensation?
    \answerNA{}
\end{enumerate}

\end{enumerate}
}
\newpage
\appendix

\section{Proofs for Section \ref{sec:adv-reg-bound}}
\label{sec:adv-regbnd-proofs}
In this section we prove Lemma \ref{lem:omwu-local}.  Throughout the section we use the notation of Lemma \ref{lem:omwu-local}: in particular, we assume that any player $i \in [m]$ follows the \Opthedge updates (\ref{eq:opt-hedge}) with step size $\eta > 0$, for an arbitrary sequence of losses $\ell_i\^1, \ldots, \ell_i\^T$.
\subsection{Preliminary lemmas}
The first few lemmas in this section pertain to vectors $P,Q \in \Delta^n$, for some $n \in \BN$; note that such vectors $P,Q$ may be viewed as distributions on $[n]$. Let $P/Q \in \BR^n$ denote the Radon-Nikodym derivative, i.e., the vector whose $j$th component is $P(j) / Q(j)$. 
\begin{lemma}
  \label{lem:chi2-close}
If $\| P/Q \|_\infty \leq A$, then $\chisq{P}{Q} \leq A \cdot \chisq{Q}{P}$.
\end{lemma}
\begin{proof}
  The lemma is immediate from the definition of the $\chi^2$ divergence:
  \begin{align*}
\chisq{P}{Q} = \sum_{j=1}^n \frac{(P(j) - Q(j))^2}{Q(j)} \leq A \cdot \sum_{j=1}^n \frac{(P(j) - Q(j))^2}{P(j)} = A \cdot \chisq{Q}{P}.
  \end{align*}
\end{proof}

It is a standard fact (though one which we do not need in our proofs) that for all $P,Q \in \Delta^{n_i}$, $\kld{P}{Q} \leq \chisq{P}{Q}$. The below lemma shows an inequality in the opposite direction when $\| P/Q \|_\infty, \| Q/P \|_\infty$ are bounded:
\begin{lemma}
  \label{lem:kl-div-lb}
  There is a constant $C$ so that the following holds. 
Suppose that for $A \leq \frac32$ we have $\left\| P/Q \right\|_\infty \leq A$ and $\left\| Q/P \right\|_\infty \leq A$. Then $(1/2 - C (A-1)) \cdot \chi^2(P;Q) \leq \KL(P;Q)$. 
\end{lemma}
\begin{proof}
There is a constant $C > 0$ so that for any $0 < \beta \leq 1/2$, for all $|x| \leq \beta$, we have
  $$
\log(1+x) \geq x - (1/2 + C\beta) x^2.
$$
Set $a = A-1$, so that $|P(j) / Q(j) - 1| \leq a$ for all $j$ by assumption. 
Then for $C' = C + 1/2$, we have
  \begin{align*}
    \KL(P; Q) &= \sum_j P(j) \log \frac{P(j)}{Q(j)} \\
              &\geq \sum_j P(j) \cdot \left( \left( \frac{P(j)}{Q(j)} - 1 \right) - (1/2 + Ca) \left( \frac{P(j)}{Q(j)} - 1 \right)^2 \right) \\
              & \geq \chi^2(P;Q) - (1/2 + Ca) \sum_j P(j) \cdot \frac{(P(j) - Q(j))^2}{Q(j)^2} \\
              & \geq \chi^2(P;Q) - \frac{1/2 + Ca}{A} \chi^2(P;Q) \\
              &\geq \frac{1/2 -aC}{1 + a} \cdot \chisq{P}{Q} \\
              &\geq \left(1/2 - a \cdot (C + 1/2)\right) \cdot \chisq{P}{Q} \\
              & = (1/2 - C'a) \cdot \chi^2(P;Q).
  \end{align*}
\end{proof}

The next lemma considers two vectors $x,x' \in \Delta^n$ which are related by a multiplicative weights-style update with loss vector $w \in \BR^n$; the lemma relates $\chisq{x'}{x}$ to $\| w \|_x^2$. 
\begin{lemma}
  \label{lem:chi2-variance}  
There is a constant $C > 0$ so that the  following holds. Suppose that $w \in \BR^n$, $\alpha > 0$,  $\| w\|_\infty \leq \alpha/2 \leq 1/C$, and $x, x' \in \Delta^n$ satisfy, for each $j \in [n]$, 
\begin{align}
  \label{eq:xxp-w}
x'(j) = \frac{x(j) \cdot \exp(w(j))}{\sum_{k \in [n]} x(k) \cdot \exp(w(k))}.
\end{align}
Then
$$
(1 - C \alpha) \cdot \Var{x}{w} \leq \chi^2(x';x) \leq (1 + C \alpha) \Var{x}{w}.
$$
\end{lemma}
\begin{proof}
Let $w' = w - \lng x, w \rng \bbone$, where $\bbone$ denotes the all-1s vector. Note that $\Var{x}{w} = \Var{x}{w'}$, and that if we replace $w$ with $w'$, (\ref{eq:xxp-w}) remains true. Moreover, $\| w' \|_\infty \leq 2 \| w \|_\infty \leq \alpha$. Thus, by replacing $w$ with $w'$, we may assume from here on that $\lng w,x \rng = 0$ and that $\| w \|\leq \alpha$. 
  
  Note that
  $$
\chi^2(x' ; x) = -1 + \sum_{i=1}^n x(i) \cdot (x'(i) / x(i))^2 = -1 + \E \left( \frac{\exp(W)}{\E \exp(W)} \right)^2,
$$
where $W$ is a random variable that takes values $w(j)$ with probability $x(j)$.  As long as $C$ is a sufficiently large constant, we have that, for all $z$ satisfying $|z| \leq \alpha$,
\begin{equation}
  \label{eq:z-rel-ps}
1 + z + (1-C\alpha) z^2 / 2 \leq \exp(z) \leq 1 + z + (1 + C\alpha)z^2 / 2.
\end{equation}
Thus, for a sufficiently large constant $C_0'$, we have, for all $z$ satisfying $|z| \leq \alpha$,
\begin{equation}
  \label{eq:exp-2z}
1 + 2z + (2 - C_0'\alpha) z^2 \leq \exp(z)^2 \leq 1 + 2z + (2 + C_0'\alpha) z^2.
\end{equation}
Moreover, since $\E W = 0$, we have from (\ref{eq:z-rel-ps}) that $1 + (1-C\alpha) \E W^2/2 \leq \E \exp(W) \leq 1 + (1+C\alpha) \E W^2/2$. For a sufficiently large constant $C_1'$ it follows that
\begin{equation}
  \label{eq:exp-exp-2z}
1 + (1-C_1'\alpha) \E W^2 \leq (\E \exp(W))^2 \leq 1 + (1+C_1'\alpha) \E W^2.
\end{equation}
Combining (\ref{eq:exp-2z}) and (\ref{eq:exp-exp-2z}) and again using the fact that $\E W = 0$, we get, for some sufficiently large constant $C''$, as long as $\alpha < 1/C_1'$,
\begin{align}
  (1 - C'' \alpha) \E W^2 \leq &  -1 + \frac{1 + (2-C_0'\alpha) \E W^2}{1 + (1 + C_1'\alpha) \E W^2} \nonumber\\
  \leq &  -1 + \frac{\E (\exp(W)^2)}{(\E \exp(W))^2} \nonumber\\
  \leq & -1 + \frac{1 + (2 + C_0'\alpha) \E W^2}{1 + (1-C_1'\alpha) \E W^2} \nonumber\\
  \leq &   (1 + C'' \alpha) \E W^2.\nonumber
\end{align}
By the assumption that $\lng w,x \rng = 0$, we have $\E W^2 = \Var{x}{w}$, and thus the above gives the desired result.
\end{proof}

We will need the following standard lemma:
\begin{lemma}[\cite{rakhlin_online_2013}, Eq.~(26)]
  \label{lem:bregman-kl}
  For any $n \in \BN$, $\ell \in \BR^n$, $y \in \Delta^n$, if it holds that $x = \argmin_{x' \in \Delta^n} \lng x', \ell \rng + \kld{x'}{y}$, then for any $z \in \Delta^n$,
  \begin{align}
\lng x - z, \ell \rng \leq \kld{z}{y} - \kld{z}{x} - \kld{x}{y}.\nonumber
  \end{align}
\end{lemma}

For $t \in [T]$, we define the vector $\til x_i\^t \in \Delta^{n_i}$ by
\begin{align}
  \label{eq:def-tilxi}
\til x_i\^t(j) := \frac{x_i\^{t}(j) \cdot \exp(-\eta \cdot (\ell_i\^t (j) - \ell_i\^{t-1}(j)))}{\sum_{k \in [n_i]} x_i\^{t}(k) \cdot \exp(-\eta \cdot (\ell_i\^t (k) - \ell_i\^{t-1}(k)))}.
\end{align}
Additionally define $\til x_i\^0 := (1/n_i, \ldots, 1/n_i)$ to be the uniform distribution over $[n_i]$.

The next lemma, Lemma \ref{lem:omwu-localnorm} is very similar to \cite[Lemma 3]{rakhlin_online_2013}, and is indeed essentially shown in the course of the proof of that lemma. Note that no boundedness assumption is placed on the vectors $\ell_i\^t$ in Lemma \ref{lem:omwu-localnorm}. For completeness we provide a full proof of the lemma.
\begin{lemma}[Refinement of Lemma 3, \cite{rakhlin_online_2013}]
  \label{lem:omwu-localnorm}
  Suppose that any player $i \in [m]$ follows the \Opthedge updates (\ref{eq:opt-hedge}) with step size $\eta > 0$, for an arbitrary sequence of losses $\ell_i\^1, \ldots, \ell_i\^T \in \BR^{n_i}$. For any vector $x^\st \in \Delta^{n_i}$, it holds that
  \begin{align}
  \hspace{-0.8cm}  \sum_{t=1}^T \lng x_i\^t - x^\st, \ell_i\^t \rng \leq  \frac{\log n_i}{\eta} + \sum_{t=1}^T \normst{x_i\^t}{ x_i\^t - \tilde x_i\^t } \sqrt{\Var{x_i\^t}{\ell_i\^t - \ell_i\^{t-1}}}
    - \frac{1}{\eta} \sum_{t=1}^T \kld{\tilde x_i\^t}{x_i\^t} - \frac{1}{\eta} \sum_{t=1}^T \kld{x_i\^t}{ \tilde x_i\^{t-1}} .\label{eq:rakhlin-refinement}
  \end{align}
\end{lemma}
\begin{proof}
  For any $x^\st \in \Delta^{n_i}$, it holds that
  \begin{align}
    \label{eq:add-subtract-mt}
\lng x_i\^t - x^\st, \ell_i\^t \rng = \lng x_i\^t - \til x_i\^t, \ell_i\^t - \ell_i\^{t-1} \rng + \lng x_i\^t - \til x_i\^t, \ell_i\^{t-1} \rng + \lng \til x_i\^t - x^\st, \ell_i\^t \rng.
  \end{align}
  For $t \in [T]$, set $c\^t = \lng x_i\^t, \ell_i\^t - \ell_i\^{t-1} \rng$. Using the definition of the dual norm and the fact $\lng x_i\^t - \til x_i\^t, \bbone \rng = 0$, we have \noah{check this re the particular choice of norm/dual norm}
  \begin{align}
    \lng x_i\^t - \til x_i\^t, \ell_i\^t - \ell_i\^{t-1} \rng = & \lng x_i\^t - \til x_i\^t, \ell_i\^t - \ell_i\^{t-1} - c\^t \bbone \rng \nonumber\\
    \leq & \normst{x_i\^t}{x_i\^t - \til x_i\^t} \cdot \norm{x_i\^t}{\ell_i\^t - \ell_i\^{t-1} - c\^t \bbone}\nonumber\\
    \leq &  \normst{x_i\^t}{x_i\^t - \til x_i\^t} \cdot \sqrt{\Var{x_i\^t}{\ell_i\^t - \ell_i\^{t-1}}}\label{eq:holder-xx}.
  \end{align}
  It is immediate from the definitions of $\til x_i\^t$ (in (\ref{eq:def-tilxi})) and $x_i\^t$ (in (\ref{eq:opt-hedge})) that for $j \in [n_i]$, \noah{check this}
  \begin{align}
x_i\^t(j) = \frac{\til x_i\^{t-1}(j) \cdot \exp(- \eta \cdot \ell_i\^{t-1}(j))}{\sum_{k \in [n_i]} \til x_i\^{t-1}(k) \cdot \exp(- \eta \cdot \ell_i\^{t-1}(k))} = \argmin_{x \in \Delta^{n_i}} \left\lng x, \eta \cdot \ell_i\^{t-1} \right\rng + \kld{x}{\til x_i\^{t-1}}\label{eq:xit-tilxitm1}
  \end{align}
  Using Lemma \ref{lem:bregman-kl} with $x = x_i\^t, \ell = \eta \ell_i\^{t-1}, y = \til x_i\^{t-1}, z = \til x_i\^t$, we obtain
  \begin{align}
\lng x_i\^t - \til x_i\^t, \ell_i\^{t-1} \rng \leq \frac{1}{\eta} \kld{\til x_i\^t}{\til x_i\^{t-1}} - \frac{1}{\eta} \kld{\til x_i\^t}{x_i\^t} - \frac{1}{\eta} \kld{x_i\^t}{\til x_i\^{t-1}}\label{eq:xxltm1}.
  \end{align}
  Next, we note that, again by (\ref{eq:def-tilxi}) and (\ref{eq:opt-hedge}), for $j \in [n_i]$,
  \begin{align}
\til x_i\^t(j) = \frac{\til x_i\^{t-1}(j) \cdot \exp(-\eta \cdot \ell_i\^t(j))}{\sum_{k \in [n_i]} \til x_i\^{t-1}(k) \cdot \exp(-\eta \cdot \ell_i\^t(k))} = \argmin_{x \in \Delta^{n_i}} \left \lng x, \eta \cdot \ell_i\^t \right\rng + \kld{x}{\til x_i\^{t-1}}\nonumber.
  \end{align}
  Using Lemma \ref{lem:bregman-kl} with $x = \til x_i\^t, \ell = \eta \ell_i\^t, y = \til x_i\^{t-1}, z = x^\st$, we obtain
  \begin{align}
\lng \til x_i\^t - x^\st, \ell_i\^t \rng \leq \frac{1}{\eta} \kld{x^\st}{\til x_i\^{t-1}} - \frac{1}{\eta} \kld{x^\st}{\til x_i\^t} - \frac{1}{\eta} \kld{\til x_i\^t}{\til x_i\^{t-1}}\label{eq:xxstlt}.
  \end{align}
  By (\ref{eq:add-subtract-mt}), (\ref{eq:holder-xx}), (\ref{eq:xxltm1}), and (\ref{eq:xxstlt}), we have
  \begin{align}
    \lng x_i\^t - x^\st, \ell_i\^t \rng
    \leq & \normst{x_i\^t}{x_i\^t - \til x_i\^t} \cdot \sqrt{\Var{x_i\^t}{\ell_i\^t - \ell_i\^{t-1}}} \nonumber\\
    & + \frac{1}{\eta} \kld{\til x_i\^t}{\til x_i\^{t-1}} - \frac{1}{\eta} \kld{\til x_i\^t}{x_i\^t} - \frac{1}{\eta} \kld{x_i\^t}{\til x_i\^{t-1}} \nonumber\\
    & + \frac{1}{\eta} \kld{x^\st}{\til x_i\^{t-1}} - \frac{1}{\eta} \kld{x^\st}{\til x_i\^t} - \frac{1}{\eta} \kld{\til x_i\^t}{\til x_i\^{t-1}}\nonumber\\
    = & \normst{x_i\^t}{x_i\^t - \til x_i\^t} \cdot \sqrt{\Var{x_i\^t}{\ell_i\^t - \ell_i\^{t-1}}} +  \frac{1}{\eta} \kld{x^\st}{\til x_i\^{t-1}} - \frac{1}{\eta} \kld{x^\st}{\til x_i\^t}\nonumber\\
    &  -\frac{1}{\eta} \kld{\til x_i\^t}{x_i\^t} - \frac{1}{\eta} \kld{x_i\^t}{\til x_i\^{t-1}}\label{eq:xxl-reg}.
  \end{align}
  The statement of the lemma follows by summing (\ref{eq:xxl-reg}) over $t \in [T]$ and using the fact that for any choice of $x^\st$, $\kld{x^\st}{x_i\^{0}} \leq \log n_i$. 
\end{proof}

\subsection{Proof of Lemma \ref{lem:omwu-local}}
Now we are ready to prove Lemma \ref{lem:omwu-local}. For convenience we restate the lemma. 
\begin{replemma}{lem:omwu-local}[restated]
  There is a constant $C > 0$ so that the following holds. 
Suppose any player $i \in [m]$ follows the \Opthedge updates (\ref{eq:opt-hedge}) with step size $0 < \eta < 1/C$, for an arbitrary sequence of losses $\ell_i\^1, \ldots, \ell_i\^T \in [0,1]^{n_i}$. Then for any vector $x^\st \in \Delta^{n_i}$, it holds that
\begin{align}
  \label{eq:adv-var-bound-2}
      \sum_{t=1}^T \lng x_i\^t - x^\st, \ell_i\^t \rng \leq \frac{\log n_i}{\eta} + \sum_{t=1}^T \left(\frac{\eta}{2} + C \eta^2\right)  \Var{x_i\^t}{\ell_i\^t - \ell_i\^{t-1}} - \sum_{t=1}^T  \frac{(1-C\eta)\eta}{2} \cdot \Var{x_i\^{t}}{\ell_i\^{t-1}}.
\end{align}
\end{replemma}
\begin{proof}
  Lemma \ref{lem:omwu-localnorm} gives that, for any $x^\st \in \Delta^{n_i}$,
    \begin{align}
      \hspace{-0.8cm}\sum_{t=1}^T \lng x_i\^t - x^\st, \ell_i\^t \rng \leq  \frac{\log n_i}{\eta} + \sum_{t=1}^T \normst{x_i\^t}{ x_i\^t - \tilde x_i\^t } \sqrt{\Var{x_i\^t}{\ell_i\^t - \ell_i\^{t-1}}}
      - \frac{1}{\eta} \sum_{t=1}^T \kld{\tilde x_i\^t}{x_i\^t} - \frac{1}{\eta} \sum_{t=1}^T \kld{x_i\^t}{ \tilde x_i\^{t-1}} .\label{eq:localnorm-starter}
    \end{align}
    Note that for any vectors $x,x' \in \Delta^{n_i}$, if there is a vector $\ell \in \BR^{n_i}$ so that for all $j \in [n_i]$, $x'(j) = \frac{x(j) \cdot \exp(\eta \cdot \ell(j))}{\sum_k x(j) \cdot \exp(\eta \cdot \ell(k))}$, we have that
$$
\exp(-2 \eta \| \ell \|_\infty) \leq \left \| \frac{x'}{x} \right\|_\infty \leq \exp(2 \eta \| \ell \|_\infty).
$$
Therefore, by (\ref{eq:def-tilxi}) and (\ref{eq:xit-tilxitm1}), respectively, we obtain that, for $\eta \leq 1/4$,
\begin{align}
\exp(-2 \eta \| \ell_i\^t - \ell_i\^{t-1} \|_\infty)  & \leq \norm{\infty}{\frac{ \til x_i\^t}{x_i\^t}} \leq \exp(2 \eta \| \ell_i\^t - \ell_i\^{t-1} \|_\infty) \leq \exp(4 \eta) \leq  1 + 8\eta \nonumber\\
\exp(-2 \eta \| \ell_i\^{t-1} \|_\infty)  & \leq \norm{\infty}{\frac{  x_i\^t}{\til x_i\^{t-1}}} \leq \exp(2 \eta \| \ell_i\^{t-1} \|_\infty) \leq\exp(2\eta) \leq 1 + 4\eta.\label{eq:xi-tilxi-infty-bnd}
\end{align}
(Above we have also used that $\| \ell_i\^t \|_\infty \leq 1$ for all $t$.) 
Thus, for $\eta \leq \frac{1}{16}$, we can apply Lemma \ref{lem:kl-div-lb} and show, for a sufficiently large constant $C_0$, 
\begin{align}
  \KL(\tilde x_i\^t; x_i\^t) & \geq \chi^2(\tilde x_i\^t; x_i\^t) \cdot (1/2 - C_0 \eta) \label{eq:kltilxi-xi} \\ 
  \KL(x_i\^t; \tilde x_i\^{t-1}) &\geq \chi^2(x_i\^t; \tilde x_i\^{t-1}) \cdot (1/2 - C_0\eta) \label{eq:klxi-tilxi}. 
\end{align}

Note also that for vectors $x,y$ we have that $\chi^2(x;y) = \left(\normst{y}{x-y} \right)^2$. By Lemma \ref{lem:chi2-variance} and (\ref{eq:def-tilxi}), we have that, for a sufficiently large constant $C_1$, as long as $\eta \leq 1/C_1$,
\begin{align}
\left( \normst{x_i\^t}{x_i\^t - \til x_i\^t} \right)^2 = \chisq{\til x_i\^t}{x_i\^t} \leq (1 + C_1\eta)\eta^2 \cdot  \Var{x_i\^t}{\ell_i\^t - \ell_i\^{t-1}}\label{eq:dualnorm-ub}
\end{align}
and 
\begin{align}
  \chi^2(\til x_i\^t; x_i\^t) \geq (1 - C_1\eta)\eta^2 \cdot \Var{x_i\^t}{\ell_i\^t - \ell_i\^{t-1}} \label{eq:chi2-lb-1}.
\end{align}

Next we lower bound $\chisq{x_i\^t}{\til x_i\^{t-1}}$ as follows, where $C_2$ denotes a sufficiently large constant: as long as $\eta \leq 1/C_2$,
\begin{align}
  \chi^2(x_i\^t; \til x_i\^{t-1}) & \geq \chi^2(\til x_i\^{t-1}; x_i\^t) \cdot \exp(-2\eta) \label{eq:flip-chi2}\\
  & \geq (1-C_2\eta) \eta^2 \cdot\Var{x_i\^t}{\ell_i\^{t-1}}\label{eq:xitilxi-var-lb},
\end{align}
where (\ref{eq:flip-chi2}) follows from Lemma \ref{lem:chi2-close} and (\ref{eq:xi-tilxi-infty-bnd}), and (\ref{eq:xitilxi-var-lb}) follows from Lemma \ref{lem:chi2-variance} and (\ref{eq:xit-tilxitm1}).

Combining (\ref{eq:localnorm-starter}), (\ref{eq:kltilxi-xi}), (\ref{eq:klxi-tilxi}), (\ref{eq:dualnorm-ub}), (\ref{eq:chi2-lb-1}), and (\ref{eq:xitilxi-var-lb}) gives that for a sufficiently large constant $C$, as long as $\eta < 1/C$, \noah{more explanation maybe...}
\begin{align*}
    \sum_{t=1}^T \lng x_i\^t - x^\st, \ell_i\^t \rng & \leq \frac{\log n_i}{\eta} + \sum_{t=1}^T (\eta/2 + C \eta^2) \cdot \Var{x_i\^t}{\ell_i\^t - \ell_i\^{t-1}} - \frac{(1-C\eta) \eta}{2} \cdot \Var{x_i\^{t}}{\ell_i\^{t-1}},
\end{align*}
as desired.
\end{proof}

\section{Proofs for Section \ref{sec:upwards-ind}}
\label{sec:upwards-ind-proofs}
In this section we give the full proof of Lemma \ref{lem:dh-bound}. In Section \ref{sec:bcr-prelims} we introduce some preliminaries. In Section \ref{sec:bcr-proof} we prove Lemma \ref{lem:fd-analytic}, the ``boundedness chain rule'' for finite differences. In Section \ref{sec:dh-bound-proof} we show how to use this lemma to prove Lemma \ref{lem:dh-bound}.

\subsection{Additional preliminaries}
\label{sec:bcr-prelims}
In this section we introduce some additional notations and basic combinatorial lemmas. Definition \ref{def:shift} introduces the \emph{shift operator} $\shf{s}{}$, which like the finite difference operator $\fd{h}{}$, maps one sequence to another sequence. 
\begin{defn}[Shift operator]
  \label{def:shift}
  Suppose $L = (L\^1, \ldots, L\^T)$ is a sequence of vectors $L\^t \in \BR^n$. For integers $s \geq 0$, the \emph{$s$-shift sequence} for the sequence $L$, denoted by $\shf{s}{L}$, is the sequence $\shf{s}{L} = (\shfs{s}{L}{1}, \ldots, \shfs{s}{L}{T-s})$, defined by $\shfs{s}{L}{t} = L\^{t+s}$ for $1 \leq t \leq T-s$. 
\end{defn}

For sequences $L = (L\^1, \ldots, L\^T)$ and $K = (K\^1, \ldots, K\^T)$ of real numbers, we will denote the \emph{product sequence} as $L \cdot K$ as the sequence of vectors $ L \cdot K := (L\^1 K\^1, \ldots, L\^T K\^T)$. Lemmas \ref{lem:prod-shift} and \ref{lem:prod-multi} below are standard analogues of the product rule for finite differences. The (straightforward) proofs are provided for completeness.
\begin{lemma}[Product rule; Eq.~(2.55) of \cite{graham_concrete_1989}] 
  \label{lem:prod-shift}
  Suppose $L = (L\^1, \ldots, L\^T)$ and $K = (K\^1, \ldots, K\^T)$ are sequences of real numbers. Then the product sequence $L \cdot K$ satisfies
  \begin{equation}
\fd{1}{(L \cdot K)} = L \cdot \fd{1}{K} + \fd{1}{L} \cdot \shf{1}{K}.\nonumber
  \end{equation}
\end{lemma}
\begin{proof}
  We compute
\begin{align*}
    \fd{1}{(L \cdot K)}^{(t)}&=L^{(t+1)}K^{(t+1)}-L^{(t)}K^{(t)}\\
    &= L^{(t+1)}K^{(t+1)}-L^{(t)}K^{(t+1)} + L^{(t)}K^{(t+1)}-L^{(t)}K^{(t)}\\
    &= (L \cdot \fd{1}{K} + \fd{1}{L} \cdot \shf{1}{K})^{(t)}.
\end{align*}
\end{proof}

\begin{lemma}[Multivariate product rule]
  \label{lem:prod-multi}
  Suppose that $m \in \BN$ and for $1 \leq i \leq m$, $L_i = (L_i\^1, \ldots, L_i\^T)$ are sequences of real numbers. Then the product sequence $\prod_{i=1}^m L_i$ satisfies
  \begin{equation}
\fd{1}{\prod_{i=1}^m L_i} = \sum_{i=1}^m \left( \prod_{i' < i} L_{i'} \right) \cdot \fd{1}{L_i} \cdot \left(\prod_{i' > i} \shf{1}{L_{i'}} \right).\nonumber
    \end{equation}
  \end{lemma}
  \begin{proof}
    We compute
\begin{align*}
    \p{\fd{1}{\prod_{i=1}^m L_i}}^{(t)}&=\prod_{i=1}^m L_i^{(t+1)}-\prod_{i=1}^m L_i^{(t)}\\
    &= \sum_{i=1}^m \p{\prod_{i' \leq i}L_{i'}^{(t+1)}\prod_{i' > i}L_{i'}^{(t)} -\prod_{i' < i}L_{i'}^{(t+1)}\prod_{i' \geq i}L_{i'}^{(t)} }\\
    &= \sum_{i=1}^m \p{\prod_{i' < i}L_{i'}^{(t+1)}\cdot \prod_{i' > i}L_{i'}^{(t)} \cdot \p{L_i^{(t+1)}-L_i^{(t)}}}\\
    &= \p{\sum_{i=1}^m \left( \prod_{i' < i} L_{i'} \right) \cdot \fd{1}{L_i} \cdot \left(\prod_{i' > i} \shf{1}{L_{i'}} \right)}^{(t)}.
\end{align*}
\end{proof}

Lemma \ref{lem:hk-function-bound} and Lemma \ref{lem:HL}, which is used in the proof of the former, are used to bound certain sums with many terms in the proof of Lemma \ref{lem:fd-analytic}. To state Lemma \ref{lem:HL} we make one definition. For positive integers $k, m$ and any $h,C > 0$, define
\begin{equation}
    R_{h,m,k,C} = \sum_{0 \leq n_1,\cdots,n_k \leq m} \p{\frac{\prod_{i=1}^k n_i^{n_i}}{h^{\sum_{i=1}^k n_i}}}^{C}\nonumber,
  \end{equation}
  where the sum is over integers $n_1, \ldots, n_k$ satisfying $0 \leq n_i \leq m$ for $i \in [k]$.  In the definition of $R_{h,m,k,C}$, the quantity $0^0$ (which arises when some $n_i = 0$) is interpreted as 1.

\begin{lemma}\label{lem:HL}
For any positive integers $k,m$ and any $h,C > 0$ so that $m\leq h/2$, $C \geq 2$, and $h \geq 8$, then 
\begin{equation}
    R_{h,m,k,C} \leq \exp\p{\frac{2k}{h^{C}}}\nonumber.
\end{equation}
\end{lemma}
\begin{proof}[Proof of Lemma \ref{lem:HL}]
We may rewrite $R_{h,m,k,C}$ and then upper bound it as follows:
\begin{align}
    R_{h,m,k,C} &= \p{\sum_{j=0}^{m} \p{\frac{j}{h}}^{Cj}}^k\nonumber\\
        &\leq \p{1+\p{\frac{1}{h}}^{C}+(m-1)\max\p{\p{\frac{2}{h}}^{2C},\p{\frac{m}{h}}^{mC}}}^k\label{eq:i-convex}\\
        &\leq \p{1+\p{\frac{1}{h}}^{C}+(h/2)\max\p{\p{\frac{2}{h}}^{2C},\p{\frac{1}{2}}^{hC/2}}}^k\nonumber
\end{align}
where (\ref{eq:i-convex}) follows since $\p{\frac{i}{h}}^{Ci}$ is convex in $i$ for $i \geq 0$, and therefore, in the interval $[2,m] \subseteq [2,h/2]$, takes on maximal values at the endpoints. We see
\begin{equation}
    (h/2)\p{\frac{2}{h}}^{2C}=\p{\frac{2}{h}}^{2C-1} \leq \p{\frac{1}{h}}^{C}\nonumber
\end{equation}
for $h \geq 8$ when $C \geq 2$. Also,
\begin{equation}
    (h/2)\p{\frac{1}{2}}^{hC/2} \leq \p{\frac{1}{h}}^{C}\nonumber
\end{equation}
for $h \geq 8$ when $C \geq 2$. (This inequality is easily seen to be equivalent to the fact that $(C+1)\log h - \frac{Ch}{2} \leq 1$, which follows from the fact that $\log h - h/2 \leq 0$ for $h \geq 8$ and $3\log h - h \leq 1$ for $h \geq 8$.)  Therefore,
\begin{align}
    R_{h,m,k,C} & \leq \p{1+\p{\frac{1}{h}}^{C}+(h/2)\max\p{\p{\frac{2}{h}}^{2C},\p{\frac{1}{2}}^{hC/2}}}^k\nonumber\\
        &\leq \p{1+2\p{\frac{1}{h}}^{C}}^k\nonumber\\
        &\leq \exp\p{\frac{2k}{h^{C}}}\nonumber.
\end{align}
\end{proof}

\begin{lemma}
  \label{lem:hk-function-bound}
  Fix integers $h \geq 0, k \geq 1$. 
For any function $\pi : [h] \ra [k]$, define, for each $i \in [k]$, $h_i(\pi) = \card{\set{q \in [h] | \pi(q) = i}}$.  Then, for any $C \geq 3$,
\begin{equation}
    \sum_{\pi : [h] \ra [k]}\frac{\prod_{i=1}^k h_i(\pi)^{Ch_i(\pi)}}{h^{Ch}} \leq  \max\left\{ k^7, (hk+1) \cdot \exp\left( \frac{2k}{h^{C-1}} \right)\right\}\label{eq:sum-pi-lemma}.
  \end{equation}
\end{lemma}

\begin{proof}
In the case that $h \leq 7$, we simply use the fact that the number of functions $\pi : [h] \ra [k]$ is $k^h \leq k^7$, and each term of the summation on the left-hand side of (\ref{eq:sum-pi-lemma}) is at most 1. In the remainder of the proof we may thus assume that $h \geq 8$. 
  
For any tuple $(h_1,\cdots,h_k)$ of non-negative integers with $\sum_{i=1}^k h_i = h$, there are ${h \choose h_1,h_2,\cdots,h_k} \leq \frac{h^h}{\prod_i h_i^{h_i}}$ (see \cite[Lemma 2.2]{csiszar_information_2004} for a proof of this inequality) functions $\pi: [h] \ra [k]$ such that $h_i(\pi)=h_i$ for all $i \in [k]$.  Combining these like terms,
\begin{align}
    \sum_{\pi : [h] \ra [k]}\frac{\prod_i h_i(\pi)^{Ch_i(\pi)}}{h^{Ch}} &\leq \sum_{\substack{h_1,\cdots,h_k \geq 0\\\sum h_i=h}} \frac{h^h}{\prod_i h_i^{h_i}} \cdot \p{\frac{\prod_i h_i^{h_i}}{h^h}}^{C}\nonumber\\
    &\leq \sum_{\substack{h_1,\cdots,h_k \geq 0\\\sum h_i=h}} \p{\frac{\prod_i h_i^{h_i}}{h^h}}^{C-1}\label{eq:sum-pi-1}.
\end{align}

We evaluate this sum in 2 cases: whether or not $h_{\max} := \max_i\{h_i\}$ is greater than $h/2$. The contribution to this sum coming from terms with $h_{\max} \leq h/2$ is 
\begin{align}
    \sum_{\substack{h_1,\cdots,h_k \geq 0\\h_1,\cdots,h_k \leq h/2\\\sum h_i=h}} \p{\frac{\prod_i h_i^{h_i}}{h^h}}^{C-1}&\leq \sum_{\substack{h_1,\cdots,h_k \geq 0\\h_1,\cdots,h_k \leq h/2\\}} \p{\frac{\prod_i h_i^{h_i}}{h^{\sum h_i}}}^{C-1}\nonumber\\
    &=R_{h,\lfloor h/2\rfloor,k,C-1}\nonumber\\
    &\leq \exp\p{\frac{2k}{h^{C-1}}}\label{eq:sum-pi-2},
\end{align}
by Lemma \ref{lem:HL}. 

We next consider the case where $h_{\max} > h/2$.  For a specific term $(h_1,\cdots,h_k)$ with $\max_i \{ h_i \} > h/2$, we know there is a unique $M \in [k]$ such that $h_M = \max_i \{ h_i \}$ since $\sum_{i=1}^k h_i = h$.  So, we can represent the contribution to the sum from this case as
\begin{align}
    \sum_{M=1}^k \sum_{\substack{h_1,\cdots,h_k \geq 0\\h_M > h/2\\\sum h_i=h}} \p{\frac{\prod_i h_i^{h_i}}{h^h}}^{C-1} &= k\sum_{\substack{h_1,\cdots,h_k \geq 0\\h_k > h/2\\\sum h_i=h}} \p{\frac{\prod_i h_i^{h_i}}{h^h}}^{C-1}\label{eq:sumprod-symmetry}\\
    &\leq k\sum_{d=0}^{\lfloor h/2\rfloor } \p{\frac{(h-d)^{h-d}}{h^{h-d}}}^{C-1}\sum_{\substack{h_1,\cdots,h_{k-1} \geq 0\\\sum h_i=d}} \p{\frac{\prod_i h_i^{h_i}}{h^d}}^{C-1}\label{eq:factor-hmax}\\
    &\leq k\sum_{d=0}^{\lfloor h/2\rfloor} \sum_{\substack{h_1,\cdots,h_{k-1} \geq 0\\h_1,\cdots,h_{k-1} \leq d}} \p{\frac{\prod_i h_i^{h_i}}{h^{\sum h_i}}}^{C-1}\nonumber\\
    &= k\sum_{d=0}^{\lfloor h/2\rfloor } R_{h,d,k-1,C-1}\nonumber\\
    &\leq kh \cdot \exp \left( \frac{2k}{h^{C-1}}\right)\label{eq:use-hl-lem},
\end{align}
where (\ref{eq:sumprod-symmetry}) follows by symmetry, (\ref{eq:factor-hmax}) follows by factoring out the contribution of $\left( \frac{h_k^{h_k}}{h^{h_k}}\right)^C$ and letting $d = h-h_k$, and (\ref{eq:use-hl-lem}) follows 
by Lemma \ref{lem:HL}.

The statement of the lemma follows from (\ref{eq:sum-pi-1}), (\ref{eq:sum-pi-2}), and (\ref{eq:use-hl-lem}). 
\end{proof}

\begin{lemma}\label{lem:softmax-ak-bound}
For $n \in \BN$, let $\xi_1, \ldots, \xi_n \geq 0$ such that $\xi_1 + \cdots + \xi_n = 1$.  For each $j\in [n]$, define $\phi_j : \BR^n \ra \BR$ to be the function
\begin{align}
\phi_j((z_1, \ldots, z_n)) = \frac{\xi_j \exp(z_j)}{\sum_{k =1}^n \xi_k \cdot \exp(z_k)}\nonumber
\end{align}
and let $P_{\phi_j}(z) = \sum_{\gamma \in \BZ_{\geq 0}^n} a_{j,\gamma} \cdot z^\gamma$ denote the Taylor series of $\phi_j$. Then for any $j \in [n]$ and any integer $k \geq 1$, 
\begin{align}
    \sum_{\gamma \in \BZ_{\geq 0}^n : \ |\gamma| = k} \card{a_{j,\gamma}} \leq \xi_j e^{k+1}\nonumber.
\end{align}
\end{lemma}

\begin{proof}
Note that, for each $j \in [n]$,
\begin{align*}
    a_{j,\gamma} = \frac{1}{\gamma_1!\gamma_2!\cdots\gamma_n!}\cdot \frac{\partial^k \phi_j(0)}{\partial z_1^{\gamma_1}\partial z_2^{\gamma_2}\cdots z_n^{\gamma_n}},
\end{align*}
and so
\begin{align*}
    \sum_{\gamma \in \BZ_{\geq 0}^n : \ |\gamma| = k} \card{a_{j,\gamma}} &=\sum_{\gamma \in \BZ_{\geq 0}^n : \ |\gamma| = k} \frac{1}{\gamma_1!\gamma_2!\cdots\gamma_n!}\cdot \card{\frac{\partial^k \phi_j(0)}{\partial z_1^{\gamma_1}\partial z_2^{\gamma_2}\cdots z_n^{\gamma_n}}}\\
    &=\frac{1}{k!}\sum_{\gamma \in \BZ_{\geq 0}^n : \ |\gamma| = k} \frac{k!}{\gamma_1!\gamma_2!\cdots \gamma_n!}\cdot \card{\frac{\partial^k \phi_j(0)}{\partial z_1^{\gamma_1}\partial z_2^{\gamma_2}\cdots z_n^{\gamma_n}}}\\
    &=\frac{1}{k!}\sum_{t \in [n]^k} \card{\frac{\partial^k \phi_j(0)}{\partial z_{t_1} \partial z_{t_2}\cdots \partial z_{t_k}}}.
\end{align*}
It is straightforward to see that the following equalities hold for any $i \in [n]$, $i \neq j$:
\begin{align*}
    \frac{\partial \phi_j}{\partial z_j} &= \phi_j(1-\phi_j)\\
    \frac{\partial \phi_j}{\partial z_i} &= -\phi_i\phi_j\\
    \frac{\partial (1-\phi_j)}{\partial z_j} &= -\phi_j(1-\phi_j)\\
    \frac{\partial (1-\phi_j)}{\partial z_i} &= \phi_i\phi_j
\end{align*}

We claim that for any $(t_1, \ldots, t_k) \in [n]^k$, we can express $\frac{\partial^k \phi_j}{\partial z_{t_1} \cdots \partial z_{t_k}}$ as a polynomial in $\phi_1,\cdots,\phi_n,(1-\phi_1),\cdots,(1-\phi_n)$ comprised of $k!$ monomials each of degree $k+1$.  We verify this by induction, first noting that after taking zero derivatives, the function $\phi_j$ is a degree-1 monomial.  Assume that for some sequence $b_1, \ldots, b_{(\ell-1)!} \in \{0,1\}$, we can express

\begin{align*}
    \frac{\partial^{\ell-1} \phi_j}{\partial z_{t_1} \cdots \partial z_{t_{\ell-1}}} = \sum_{f = 1}^{(\ell-1)!} (-1)^{b_f} \prod_{d = 0}^{\ell-1} m_{f,d}
\end{align*}
where each $m_{f,d} \in \set{\phi_1,\cdots,\phi_n,(1-\phi_1),\cdots,(1-\phi_n)}$.  We see that for each $f$, there is some sequence of bits $b_{f,0}, \ldots, b_{f,\ell-1} \in \{0,1\}$ so that 
\begin{align}
  \label{eq:fact-tree-prod-rule}
    \frac{\partial}{\partial z_{t_\ell} } \prod_{d = 0}^{\ell-1} m_{f,d} = \sum_{d=0}^{\ell-1} (-1)^{b_{f,d}} \cdot m_{f,0}\cdots m_{f,d}' \cdots m_{f,d,\ell}
\end{align}
where we define, for each $0 \leq d \leq \ell-1$,
\begin{align*}
    m_{f,d}'\text{ and }m_{f,d,\ell} = \begin{cases}
        m_{f,d}\text{ and }\phi_{t_\ell} &\text{ if }m_{f,d} = \phi_i \text{ with }i \ne t_\ell\\
        m_{f,d}\text{ and }(1-\phi_{t_\ell}) &\text{ if }m_{f,d} = \phi_{t_{\ell}}\\
        (1-m_{f,d})\text{ and }\phi_{t_\ell} &\text{ if }m_{f,d} = 1-\phi_i \text{ with }i \ne t_\ell\\
        (1-m_{f,d})\text{ and }(1-\phi_{t_\ell}) &\text{ if }m_{f,d} = 1-\phi_{t_{\ell}}.
    \end{cases}
\end{align*}

Thus, $\frac{\partial^{\ell} \phi_j}{\partial z_{t_1} \cdots \partial z_{t_{\ell}}}$ can be expressed as a sum of $\ell!$ monomials of degree $(\ell+1)$, completing the inductive step.\\



This inductive argument also demonstrates a bijection between the $k!$ monomials of $\frac{\partial^k \phi_j}{\partial z_{t_1} \cdots \partial z_{t_k}}$ and a combinatorial structure that we call \emph{factorial trees}.  Formally, we define a factorial tree to be a directed graph on vertices $\set{0,1,\cdots,k}$ such that each vertex $i\ne 0$ has a single incoming edge from one of the vertices in $[0,i-1]$. (For a non-negative integer $i$, we write $[0,i] := \{ 0, 1, \ldots, i\}$.)  For a factorial tree $f$, let $p_f(\ell) \in [0,\ell-1]$ denote the parent of a vertex $\ell$.  A particular factorial tree $f$ represents the monomial that was generated by choosing the $p_f(\ell)^{\text{th}}$ term in (\ref{eq:fact-tree-prod-rule}) for derivation when taking the derivative $\frac{\partial}{\partial z_{t_\ell}}$, for each $\ell \in [k]$. (See Figure \ref{fig:fac-tree} for an example.)

\begin{figure}[h]
\includegraphics[width=0.75\textwidth]{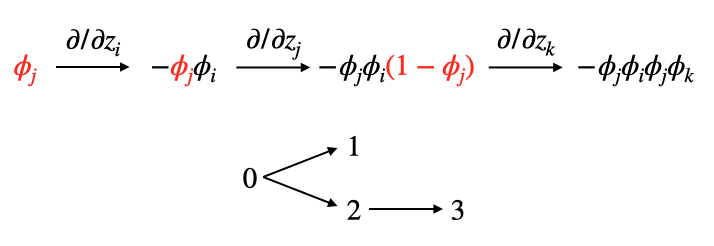}
\centering
\caption{A monomial $-\phi_j\phi_i\phi_j\phi_k$ of $\frac{\partial^3 \phi_j}{\partial z_i \partial z_j \partial z_k}$ and its corresponding factorial tree}
\label{fig:fac-tree}
\end{figure}


Each of the $k!$ monomials comprising $\frac{\partial^k \phi_j}{\partial z_{t_1} \cdots \partial z_{t_k}}$ is a product of $k+1$ terms corresponding to indices $j,t_1,\cdots,t_k$ (i.e., the first term in the product is either $\phi_j$ or $1-\phi_j$, the second term is either $\phi_{t_1}$ or $1-\phi_{t_1}$, and so on).  We say that a term corresponding to index $i \in [n]$ is \emph{perturbed} if it is  $(1-\phi_i)$ (as opposed to $\phi_i$).  From our construction, we see that the $\ell^{\text{th}}$ term is perturbed if $t_\ell = t_{p_f(\ell)}$ and there is no $\ell'$ such that $p_f(\ell') = \ell$.  That is, $\ell$ is a leaf in the corresponding factorial tree $f$ and the parent of $\ell$ corresponds to the same index as $\ell$. One can think of $t_1,\cdots,t_k$ as a coloring of all the vertices of the factorial tree with $n$ colors, except the root of the tree (vertex $0$) which has fixed color $j$.  Then, we can say the $\ell^{\text{th}}$ term is perturbed if and only if $\ell$ is a leaf with the same color as its parent.  We call such a leaf a \emph {petal}. For $t \in [n]$, we let $P_{f,t} \subseteq [k]$ be the set of petals on tree $f$ with color $t$, $L_{f} \subseteq [k]$ be the set of leaves of tree $f$, and $B_{f} = [k] \setminus L_f$ be the set of all non-leaves other than the fixed-color root.  Therefore,

\begin{align*}
    \sum_{\gamma \in \BZ_{\geq 0}^n : \ |\gamma| = k} \card{a_{j,\gamma}}&= \frac{1}{k!}\sum_{t \in [n]^k} \card{\frac{\partial^k \phi_j(0)}{\partial z_{t_1} \cdots \partial z_{t_k}}}\\
    & \leq \frac{1}{k!}\sum_{t \in [n]^k}\sum_{f}\prod_{\ell=0}^{k}(\phi_{t_\ell}(0) \cdot \BO[\ell \not \in P_{f,t}]+ (1-\phi_{t_\ell}(0)) \cdot \BO[\ell \in P_{f,t}])\\
    \intertext{(where we let $t_0=j$ for notational convenience)}
    & = \frac{1}{k!}\sum_{t \in [n]^k}\sum_{f}\prod_{\ell=0}^{k}(\xi_{t_\ell} \cdot \BO[\ell \not \in P_{f,t}]+ (1-\xi_{t_\ell}) \cdot \BO[\ell \in P_{f,t}])\\
    & = \frac{1}{k!}\sum_{f}\sum_{t_{B_f} \in [n]^{B_f}}\sum_{t_{L_f} \in [n]^{L_f}}\prod_{\ell=0}^{k}(\xi_{t_\ell} \cdot \BO[\ell \not \in P_{f,t}]+ (1-\xi_{t_\ell}) \cdot \BO[\ell \in P_{f,t}]),
\end{align*}
where in the last step we decompose, for each factorial tree $f$, $t \in [n]^k$ into the tuple of indices $t_{B_f} \in [n]^{B_f}$ corresponding to the non-leaves $B_f$, and the tuple of indices $t_{L_f} \in [n]^{L_f}$ corresponding to the leaves $L_f$.

We note that, fixing tree $f$ and the colors of all non-leaves $t_B$,

\begin{align*}
    &\sum_{t_{L_f} \in [n]^{L_f}}\prod_{\ell \in L_f}(\xi_{t_\ell} \cdot \BO[\ell \not \in P_{f,t}]+ (1-\xi_{t_\ell}) \cdot \BO[\ell \in P_{f,t}])\\
    &=\prod_{\ell \in L_f} \p{\sum_{t_\ell \in [n]} \xi_{t_\ell} \cdot \BO[t_\ell \ne t_{p_f(\ell)}]+ (1-\xi_{t_\ell}) \cdot \BO[t_\ell = t_{p_f(\ell)}]}\\
    &=\prod_{\ell \in L_f} \p{2-2\xi_{t_{p_f(\ell)}}}\\
    &\leq 2^{|L_f|}
\end{align*}

And so,

\begin{align*}
    & \frac{1}{k!}\sum_{f}\sum_{t_{B_f} \in [n]^{B_f}}\sum_{t_{L_f} \in [n]^{L_f}}\prod_{\ell=0}^{k}(\xi_{t_\ell} \cdot \BO[\ell \not \in P_{f,t}]+ (1-\xi_{t_\ell}) \cdot \BO[\ell \in P_{f,t}])\\
    &\leq \frac{1}{k!}\sum_{f}2^{|L_f|}\sum_{t_{B_f} \in [n]^{B_f}}\prod_{\ell \in B_f \cup \set{0}}(\xi_{t_\ell} \cdot \BO[\ell \not \in P_{f,t}]+ (1-\xi_{t_\ell}) \cdot \BO[\ell \in P_{f,t}])\\
    &= \frac{1}{k!}\sum_{f}2^{|L_f|}\sum_{t_{B_f} \in [n]^{B_f}}\prod_{\ell \in B_f \cup \set{0}}\xi_{t_\ell}\\
    \intertext{(as no non-leaf can ever be a petal)}
    &= \frac{\xi_j}{k!}\sum_{f}2^{|L_f|}\prod_{\ell \in B_f}\p{\sum_{t_{\ell} \in [n]}\xi_{t_\ell}}\\
    &= \frac{\xi_j}{k!}\sum_{f}2^{|L_f|} = \xi_j\EE_{f \sim \cU(\cF)}\ps{2^{|L_f|}}
\end{align*}
where $\cF$ is the set of all factorial trees and $\cU(\cF)$ is the uniform distribution over $\cF$.  For a specific vertex $\ell \in [0,k]$, we note that $\ell \in L_f$ if and only if it is not the parent of any vertex $\ell+1,\cdots,k$.  So, \begin{equation}
    \Pr_{f \sim \cU(\cF)}\ps{\ell \in L_f} = \prod_{i=\ell+1}^{k} \frac{i-1}{i} = \frac{\ell}{k}
\end{equation}
We will show via induction that, for any vertex set $S \subseteq [0,k]$
\begin{equation}\label{eq:cor-leaves}
     \Pr_{f \sim \cU(\cF)}\ps{S \subseteq L_f} \leq \prod_{\ell \in S} \frac{\ell}{k}
\end{equation}
Having established the base case for every $S$ with $|S|=1$, we assume (\ref{eq:cor-leaves}) holds for all $S$ with $|S| < s$.  For any set of $s$ vertices $V$, consider an arbitrary partition of $V$ into two sets $S \cup T = V$ with $|S|,|T|<s$.  We see

\begin{align*}
    \Pr_{f \sim \cU(\cF)}\ps{V \subseteq L_f} &= \prod_{c=1}^{k} \Pr\ps{p_f(c) \not \in V}\\
    &= \prod_{c=1}^{k} \Pr\ps{p_f(c) \not \in S}\Pr\ps{p_f(c) \not \in T|p_f(c) \not \in S}\\
    &\leq \prod_{c=1}^{k} \Pr\ps{p_f(c) \not \in S}\Pr\ps{p_f(c) \not \in T}\\
    &=\Pr\ps{S \subseteq L_f}\Pr\ps{T \subseteq L_f}\\
    &\leq \prod_{\ell \in V} \frac{\ell}{k}
\end{align*}
by the inductive hypothesis, as desired.  Thus, $\Pr\ps{|L_f| \geq s} \leq \sum_{S:|S| = s} \Pr\ps{S \subseteq L_f}$ is at most the $s^{\text{th}}$ coefficient of the polynomial
\begin{equation*}
    R(x) = \prod_{\ell = 0}^k \p{1+\frac{\ell}{k}x}
\end{equation*}
and so
\begin{align*}
    \EE_{f \sim \cU(\cF)}\ps{2^{|L_f|}} &\leq \sum_{s=0}^{k} 2^s\Pr[|L_f| \geq s]\\
    &\leq R(2)\\
    &\leq e^{\sum_{\ell=0}^k 2\ell/k} = e^{k+1}
\end{align*}
and
\begin{equation*}
    \sum_{\gamma \in \BZ_{\geq 0}^n : \ |\gamma| = k} \card{a_{j,\gamma}} \leq \xi_j e^{k+1},
\end{equation*}
as desired.
\end{proof}

\begin{lemma}\label{lem:prodmax-ak-bound}
Let $\phi_1,\cdots,\phi_m$ be softmax-type functions.  That is, for each $\phi_i$, there is some $j_i \in [n]$ and indices $\xi_{i1}, \ldots, \xi_{in}$ such that
\begin{align*}
\phi_i((z_1, \ldots, z_n)) = \frac{\exp(z_{j_i})}{\sum_{k =1}^n \xi_{ik} \cdot \exp(z_k)}\nonumber
\end{align*}
where $\xi_{i1} + \cdots + \xi_{in} = 1$ for all $i$.  Let $P(z) = \sum_{\gamma \in \BZ_{\geq 0}^n} a_\gamma z^\gamma$ denote the Taylor series of $\prod_i \phi_i$. Then for any integer $k$, 
\begin{align*}
    \sum_{\gamma \in \BZ_{\geq 0}^n : \ |\gamma| = k} \card{a_\gamma} \leq (e^3m)^{k}.
\end{align*}
\end{lemma}

\begin{proof}
Letting $P_i(z) = \sum_{\gamma \in \BZ_{\geq 0}^n} a_{i,\gamma} z^{\gamma}$ denote the Taylor series of $\phi_i$ for all $i$, we have $P(z) = \prod_i P_i(z)$ and therefore

\begin{align*}
    \sum_{\gamma \in \BZ_{\geq 0}^n : \ |\gamma| = k} \card{a_\gamma} \leq \sum_{\substack{k_1,\cdots,k_m \in \BZ_{\geq 0}\\ \sum k_i = k}} \prod_i \sum_{\gamma \in \BZ_{\geq 0}^n : \ |\gamma| = k_i} \card{a_{i,\gamma} }\\
\end{align*}
We have that $\sum_{ |\gamma| = k_i} \card{a_{i,\gamma}} \leq e^{2k_i}$ for all $k_i$ since, for $k_i=0$, $a_{i,0} = \phi_i(0) = 1$, and for $k_i \geq 1$, 
\begin{equation}\label{eq:applyingLem}
    \sum_{|\gamma| = k_i} \card{a_{i,\gamma}} \leq \frac{\xi_{ij}}{\xi_{ij}} e^{k_i+1} \leq e^{2k_i}
\end{equation}
from Lemma \ref{lem:softmax-ak-bound}. Note that the softmax-type functions discussed in Lemma \ref{lem:softmax-ak-bound} have a $\xi_{ij}$ term in the numerator, while those discussed here do not.  This accounts for the extra $\xi_{ij}$ term that appears in equation \eqref{eq:applyingLem}. Thus,
\begin{align*}
    \sum_{\substack{k_1,\cdots,k_m \in \BZ_{\geq 0}\\ \sum k_i = k}} \prod_i \sum_{\gamma \in \BZ_{\geq 0}^n : \ |\gamma| = k_i} \card{a_{i,\gamma} } &\leq \sum_{\substack{k_1,\cdots,k_m \in \BZ_{\geq 0}\\ \sum k_i = k}} e^{2k}\\
    &= e^{2k} {m+k-1 \choose k}\\
    &\leq e^{2k} \p{\frac{e(m+k-1)}{k}}^k\\
    &= (e^{3}m)^k
\end{align*}
as desired.
\end{proof}

\begin{lemma}
  \label{lem:softmax-roc}
Let $\phi((z_1, \ldots, z_n)) = \frac{\exp(z_j)}{\sum_{k=1}^n \xi_k \exp(z_k)}$ be any softmax-type function. Then the radius of convergence of the Taylor series of $\phi$ at the origin is at least 1.
\end{lemma}
\begin{proof}
  For a complex number $z$, write $\Re(z), \Im(z)$ to denote the real and imaginary parts, respectively, of $z$.  Note that for any $\zeta_1, \ldots, \zeta_n \in \BC$ with $|\zeta_k| \leq \pi/3$ for all $k \in [n]$, we have
  $$
\Re(\exp(\zeta_k)) \geq  \cos(\pi/3) \cdot \exp(-\pi/3)  > 1/10,
$$
and thus $\left| \sum_{k=1}^n \xi_k \cdot \exp(\zeta_k) \right| \geq 1/10$. Moreover, for any such point $\zeta = (\zeta_1, \ldots, \zeta_n)$, it holds that $ |\exp(\zeta_j)| \leq \exp(\pi/3) < 3$. It then follows that for such $\zeta$ we have $|\phi(z)| \leq 30$. In particular, $\phi$ is holomorphic on the region $\{ \zeta : |\zeta_k| \leq \pi/3 \ \  \forall k \in [n] \}$.

Fix any $\gamma \in \BZ_{\geq 0}^n$, and let $k = |\gamma|$. By the multivariate version of Cauchy's integral formula,
\begin{align}
 \left| \frac{d^\gamma}{dz^\gamma} \phi(z)\right| = &\left| \frac{\gamma!}{(2\pi i)^n} \int_{|\zeta_1-z_1| = \pi/3} \cdots \int_{|\zeta_n - z_n| = \pi/3} \frac{\phi(\zeta_1, \ldots, \zeta_n)}{(\zeta_1-z_1)^{\gamma_1+1} \cdots (\zeta_n - z_n)^{\gamma_n+1}} d\zeta_1 \cdots d\zeta_n\right|\nonumber\\
  \leq & \frac{30 \gamma!}{(\pi/3)^{k+n}} \leq \frac{30 \gamma!}{(\pi/3)^k}\nonumber.
\end{align}
The power series of $\phi$ at $\bbzero$ is defined as $P_\phi(z) = \sum_{\gamma \in \BZ_{\geq 0}^n} a_\gamma \cdot z^\gamma$, where $a_\gamma = \frac{1}{\gamma!} \frac{d^\gamma}{dz^\gamma} \phi(\bbzero)$. For any $\gamma \in \BZ_{\geq 0}^n$ with $k = |\gamma|$, we have $|a_\gamma|^{1/k} \leq (30 / (\pi/3)^k)^{1/k} = (30)^{1/k} \cdot 3/\pi$, which tends to $3/\pi < 1$ as $k \ra \infty$. Thus, by the (multivariate version of the) Cauchy-Hadamard theorem, the radius of convergence of the power series of $\phi$ at $\bbzero$ is at least $\pi/3 \geq 1$. \noah{check this}
\end{proof}

\subsection{Proof of Lemma \ref{lem:expand-pow-seq}}
\label{sec:bcr-proof}
In this section prove Lemma \ref{lem:expand-pow-seq}, which, as explained in Section \ref{sec:upwards-ind}, is an important ingredient in the proof of Lemma \ref{lem:fd-analytic}. The detailed version of Lemma \ref{lem:expand-pow-seq} is presented below; it includes several claims which are omitted for simplicity in the abbreviated version in Section \ref{sec:upwards-ind}.

\begin{replemma}{lem:expand-pow-seq}[Detailed]
  Fix any integer $h \geq 0$, a multi-index $\gamma \in \BZ_{\geq 0}^n$ and set $k = |\gamma|$. For each of the $k^h$ functions $\pi : [h] \ra [k]$, and for each $r \in [k]$, there are integers $h'_{\pi,r} \in \{0, 1, \ldots, h\}$, $t'_{\pi,r} \geq 0$, and $j'_{\pi,r} \in [n]$, so that the following holds. For any sequence $L\^1, \ldots, L\^T \in \BR^n$ of vectors, it holds that 
  \begin{align}
\fd{h}{L^\gamma} = \sum_{\pi : [h] \ra [k]} \prod_{r=1}^k \shf{t'_{\pi,r}}{\fd{h'_{\pi,r}}{(L(j'_{\pi,r}))}}\label{eq:prod-expand-formal}.
  \end{align}
  Moreover, the following properties hold:
\begin{enumerate}
\item \label{it:distr-si} For each $\pi$ and $r \in [k]$, $h'_{\pi,r} = | \{ q \in [h] : \pi(q) = r \}|$. In particular, $\sum_{r=1}^k h'_{\pi,r} = h$. 
\item \label{it:tpr-hpr} For each $\pi$ and $r \in [k]$, it holds that $0 \leq t'_{\pi,r} + h'_{\pi,r} \leq h$. 
  \item \label{it:j-histogram} For each $\pi$, $r \in [k]$, and $j \in [n]$, $\gamma_j = | \{ r \in [k] : j'_{\pi,r} = j \} |$. 
\end{enumerate}
\end{replemma}
\begin{proof}[Proof of Lemma \ref{lem:expand-pow-seq}]
  We use induction on $h$. First note that in the case $h = 0$ and for any $k \geq 0$, we have that $\fds{h}{L^\gamma}{t} = (L\^t)^\gamma$, and so for the unique function $\pi : \emptyset \ra [k]$, for all $r \in [k]$, we may take $t'_{\pi,r} = 0$, $h'_{\pi,r} = 0$, and ensure that for each $j \in [n]$ there are $\gamma_j$ values of $r$ so that $j'_{\pi,r} = j$. 

  Now fix any integer $h > 0$, and suppose the statement of the claim holds for all $h' < h$. 
We have that
  \begin{align}
    & \fd{h}{L^\gamma}{} \nonumber\\
    =& \fd{1}{\fd{h-1}{L^\gamma}}{} \nonumber\\
    =& \fd{1}{\sum_{\pi: [h-1] \ra [k]} \prod_{r=1}^k \shf{t'_{\pi,r}}{\fd{h'_{\pi,r}}{L(j'_{\pi,r})}}}{} \nonumber\\
    \label{eq:use-prod-lemma}
    =& \sum_{\pi : [h-1] \ra [k]} \sum_{r=1}^k \fd{1}{\shf{t'_{\pi,r}}{\fd{h'_{\pi,r}}{L(j'_{\pi,r})}}}  \cdot \prod_{r'=1 }^{r-1} \shf{t'_{\pi,r'}}{\fd{h'_{\pi,r'}}{L(j'_{\pi,r'})}} \cdot  \prod_{r'=r+1}^k \shf{t'_{\pi,r'}+1}{\fd{h'_{\pi,r'}}{L(j'_{\pi,r'})}} \\
    \label{eq:ed-commute}
        =& \sum_{\pi : [h-1] \ra [k]} \sum_{r=1}^k \shf{t'_{\pi,r}}{\fd{h'_{\pi,r}+1}{L(j'_{\pi,r})}}  \cdot \prod_{r'=1 }^{r-1} \shf{t'_{\pi,r'}}{\fd{h'_{\pi,r'}}{L(j'_{\pi,r'})}} \cdot  \prod_{r'=r+1}^k \shf{t'_{\pi,r'}+1}{\fd{h'_{\pi,r'}}{L(j'_{\pi,r'})}} .
  \end{align}
  where (\ref{eq:use-prod-lemma}) uses Lemma \ref{lem:prod-multi} and (\ref{eq:ed-commute}) uses the commutativity of $\shf{t'}{}$ and $\fd{1}{}$. For each $\pi : [h-1] \ra [k]$, we construct $k$ functions $\pi_1, \ldots, \pi_k : [h] \ra [k]$, defined by $\pi_r(q) = \pi(q)$ for $q < h$, and $\pi_r(h) = r$ for $r \in [k]$. Next, for $r, r' \in [k]$, we define the quantities $h'_{\pi_r,r'}, t'_{\pi_r,r'}, j'_{\pi_r,r'}$ as follows:
  \begin{itemize}
  \item Set $h'_{\pi_r,r'} = h'_{\pi,r'}$ if $r \neq r'$, and $h'_{\pi_r,r} = h'_{\pi,r}+1$.
  \item Set $t'_{\pi_r,r'} = t'_{\pi,r'}$ if $r' \leq r$, and $t'_{\pi_r,r'} = t'_{\pi,r'}+1$ if $r' > r$.
  \item Set $j'_{\pi_r,r'} = j'_{\pi,r'}$.
  \end{itemize}
  By (\ref{eq:ed-commute}) and the above definitions, we have
  \begin{align}
\fd{h}{L^\gamma} = \sum_{\pi : [h] \ra [k]} \prod_{r'=1}^k \shf{t'_{\pi,r}}{\fd{h'_{\pi,r}}{L(j'_{\pi,r})}},\nonumber
  \end{align}
  thus verifying (\ref{eq:prod-expand-informal}) for the value $h$.
  
 Finally we verify that items \ref{it:distr-si} through \ref{it:j-histogram} in the lemma statement hold.  The definition of $h'_{\pi_r,r'}$ above together with the inductive hypothesis ensures that for all $r,r' \in [k]$, $h'_{\pi_r,r'} = |\{ q \in [h] : \pi_r(q) = r' \}|$, thus verifying item \ref{it:distr-si} of the lemma statement. Since $h'_{\pi_r,r'} + t'_{\pi_r,r'} \leq h'_{\pi,r} + t'_{\pi,r} + 1$ for all $r,r'$, it follows from the inductive hypothesis that $0 \leq h'_{\pi_r,r'} + t'_{\pi_r,r'} \leq h$; this verifies item \ref{it:tpr-hpr}. Finally, note that for any $j \in [n]$ and $r \in [k]$, $\{ r' \in [k] : j'_{\pi,r'} = j \} = \{ r' \in [k] : j'_{\pi_r,r'} = j\}$, and thus item \ref{it:j-histogram} follows from the inductive hypothesis.  
\end{proof}

\subsection{Proof of Lemma \ref{lem:fd-analytic}}
\label{sec:bcr-proof-2}
In this section we prove Lemma \ref{lem:fd-analytic}. To introduce the detailed version of the lemma we need the following definition. 
Suppose $\phi : \BR^n \ra \BR$ is a real-valued function that is real-analytic in a neighborhood of the origin. For real numbers $Q, R > 0$, we say that $\phi$ is \emph{$(Q, R)$-bounded} if the Taylor series of $\phi$ at $\bbzero$, denoted $P_\phi(z_1, \ldots, z_n) = \sum_{\gamma \in \BZ_{\geq 0}^n} a_\gamma z^\gamma$, satisfies, for each integer $k \geq 0$, $\sum_{\gamma \in \BZ_{\geq 0}^n : |\gamma| = k} |a_\gamma| \leq Q \cdot R^k$. In the statement of Lemma \ref{lem:fd-analytic} below, the quantity $0^0$ is interpreted as 1 (in particular, $(h')^{B_0 h'} = 1$ for $h' = 0$).
\begin{replemma}{lem:fd-analytic}[``Boundedness chain rule'' for finite differences; detailed]
 Suppose that $h,n \in \BN$,  $\phi : \BR^n \ra \BR$ is a $(Q,R)$-bounded function so that the radius of convergence of its power series at $\bbzero$ is at least $\nu>0$, and $L = (L\^1, \ldots, L\^T) \in \BR^n$ is a sequence of vectors satisfying $\| L\^t \|_\infty \leq \nu$ for $t \in [T]$. 
 Suppose for some $\alpha \in (0,1)$, for each $0 \leq h' \leq h$ and $t \in [T-h']$, it holds that $\| \fd{h'}{L}\^t \|_\infty \leq \frac{1}{B_1} \cdot \alpha^{h'} \cdot (h')^{B_0 h'}$ for some $B_1 \geq 2e^2 R, B_0 \geq 3$. 
 Then for all $t \in [T-h]$, 
  \begin{align*}
| \fds{h}{(\phi \circ L)}{t}|\leq \frac{12RQ e^2}{B_1} \cdot \alpha^h \cdot h^{B_0h+1 }. 
  \end{align*}
\end{replemma}
\begin{proof}[Proof of Lemma \ref{lem:fd-analytic}]
Note that the $h$th order finite differences of a constant sequence are identically 0 for $h \geq 1$, so by subtracting $\phi(\bbzero)$ from $\phi$, we may assume without loss of generality that $\phi(\bbzero) = 0$. (Here $\bbzero$ denotes the all-zeros vector.) 
  
By assumption, the radius of convergence of the power series of $\phi$ at the origin is at least $\nu$, and so for each $\gamma \in \BZ_{\geq 0}^n$, there is a real number $a_\gamma$ so that for $z = (z_1, \ldots, z_n)$ with $|z_j| \leq \nu$ for each $j$,
  \begin{equation}
    \label{eq:phi-ps}
\phi(z) = \sum_{k \in \BN, \gamma \in \BZ_{\geq 0}^n : \ |\gamma| = k} a_\gamma z^\gamma.
\end{equation}
Let $A_k := \sum_{\gamma \in \BZ_{\geq 0}^n : |\gamma| = k} |a_\gamma|$; by the assumption that $\phi$ is $(Q,R)$-bounded, we have that $A_k \leq Q \cdot R^k$ for all $k \in \BN$. 

For $\gamma \in \BZ_{\geq 0}^n$, recall that $L^\gamma$ denotes the sequence $((L^\gamma)\^1, \ldots, (L^\gamma)\^T)$, defined by $(L^\gamma)\^t = (L\^t(1))^{\gamma_1} \cdots (L\^t(n))^{\gamma_n}$. Then since $\|L\^t\|_\infty \leq \nu$ for all $t \in [T]$, we have that, for $t \in [T-h]$, $\fds{h}{(\phi \circ L)}{t} = \sum_{\gamma \in \BZ_{\geq 0}^n} a_\gamma \cdot \fds{h}{L^\gamma}{t}$.

We next upper bound the quantities $|\fds{h}{L^\gamma}{t}|$. To do so, fix some $\gamma \in \BZ_{\geq 0}^n$, and set $k = |\gamma|$. 
For each function $\pi : [h] \ra [k]$ and $r \in [k]$, recall the integers  $h'_{\pi,r} \in \{0, 1, \ldots, h \}$, $t'_{\pi,r} \geq 0$, $j'_{\pi,r} \in [n]$ defined in Lemma \ref{lem:expand-pow-seq}.
By assumption it holds that for each $t \in [T-h]$, each $h' \leq h$, each $0 \leq t' \leq h$, $|\fds{h'}{L(j)}{t+t'}| \leq \frac{1}{B_1} \cdot \alpha^{h'} \cdot (h')^{B_0 h' }$. 
 It follows that for each $t \in [T-h]$ and function $\pi : [h] \ra [k]$,
\begin{align}
  & \left| \prod_{r=1}^k \shfs{t'_{\pi,r}}{\fd{h'_{\pi,r}}{L(j'_{\pi,r})}}{t}\right|
  \leq  \prod_{r=1}^k \frac{1}{B_1} \cdot \alpha^{h'_{\pi,r}} \cdot (h'_{\pi,r})^{B_0 h'_{\pi,r}}
  = \frac{\alpha^h}{B_1^k} \cdot \prod_{r=1}^k (h'_{\pi,r})^{B_0 h'_{\pi,r} }\nonumber,
\end{align}
where the last equality uses that $\sum_{r=1}^k h'_{\pi,r} = h$ (item \ref{it:distr-si} of Lemma \ref{lem:expand-pow-seq}). Then by Lemma \ref{lem:expand-pow-seq}, we have: 
{
\begin{align}
  \left| \fds{h}{L^\gamma}{t} \right| 
  \leq & \sum_{\pi : [h] \ra [k]} \left| \prod_{r=1}^k \shfs{t'_{\pi,r}}{\fd{h'_{\pi,r}}{L(j'_{\pi,r})}}{t}\right| \nonumber\\
  \leq &\frac{\alpha^h}{B_1^k}  \sum_{\pi : [h] \ra [k]} \prod_{r=1}^k (h'_{\pi,r})^{B_0 h'_{\pi,r} }   \nonumber\\
  \leq & \frac{\alpha^h}{B_1^k} \cdot h^{B_0 h}\max\left\{ k^7, (hk+1) \cdot \exp\left( \frac{2k}{h^{B_0-1}} \right)\right\} 
         \label{eq:use-ccs},
\end{align}
}
where (\ref{eq:use-ccs}) follows from Lemma \ref{lem:hk-function-bound}, the fact that $B_0 \geq 3$, and that $h'_{\pi,r} = | \{ q \in [h] : \pi(q) = r \}|$ (item \ref{it:distr-si} of Lemma \ref{lem:expand-pow-seq}). 

We may now bound the order-$h$ finite differences of the sequence $\phi \circ L$ as follows: for $t \in [T-h]$, 
\begin{align}
|  \fds{h}{(\phi \circ L)}{t} |\leq& \sum_{\gamma \in \BZ_{\geq 0}^n} |a_\gamma| \cdot\left| \fds{h}{L^\gamma}{t}\right|\nonumber\\
  \leq & \alpha^h \cdot h^{B_0 h+1}   \sum_{\gamma \in \BZ_{\geq 0}^n} |a_\gamma| \cdot B_1^{-|\gamma|} \cdot \max \left\{ |\gamma|^7, (|\gamma|+1) \cdot \exp \left( \frac{2|\gamma|}{h^{B_0-1}} \right) \right\} \label{eq:use-dh-ub}\\
  \leq & \alpha^h \cdot h^{B_0h+1}  \cdot \sum_{k \in \BN} A_k \cdot  B_1^{-k} \cdot \left( k^7 + 2k \cdot \exp(2k/h^{B_0-1}) \right) \nonumber\\
  \leq & \alpha^h \cdot h^{B_0h+1} \cdot Q \cdot \left( \sum_{k \in \BN} k^7 \cdot  (R/B_1)^k  + \sum_{k \in \BN} 2k \cdot (R/B_1)^k \cdot e^{2k} \right) \label{eq:use-ak-ek1-bound}\\
  \leq & \frac{2RQe^2}{B_1} \cdot \alpha^h \cdot h^{B_0h+1} \cdot \left(\sum_{k \in \BN}  k^7 \cdot (2e^2)^{-k} + 2 \sum_{k \in \BN}k \cdot 2^{-k} \right) \label{eq:use-B1} \\
  =& \frac{12RQe^2}{B_1} \cdot \alpha^h \cdot h^{B_0h+1}\nonumber.
\end{align}
where (\ref{eq:use-dh-ub}) uses (\ref{eq:use-ccs}), (\ref{eq:use-ak-ek1-bound}) uses the bound $A_k \leq QR^k$, and (\ref{eq:use-B1}) uses the assumption $B_1 \geq 2e^2 R$. This gives the desired conclusion of the lemma.
\end{proof}

\subsection{Proof of Lemma \ref{lem:dh-bound}}
\label{sec:dh-bound-proof}
In this section we prove Lemma \ref{lem:dh-bound}. The detailed version of Lemma \ref{lem:dh-bound} is stated below.
\begin{replemma}{lem:dh-bound}[Detailed]
  Fix a parameter $\alpha \in \left(0, \frac{1}{H+3} \right)$. 
  If all players follow \Opthedge updates with step size $\eta \leq \frac{\alpha}{36 e^5 m}$, then for any player $i \in [m]$, integer $h$ satisfying $0 \leq h \leq H$, time step $t \in [T-h]$, it holds that
  $$
\| \fds{h}{\ell_i}{t} \|_\infty \leq \alpha^h \cdot h^{3h+1}.
$$
\end{replemma}
\begin{proof}
  We have that for each agent $i \in [m]$, each $t \in [T]$, and each $a_i \in [n_i]$, $\ell_i\^t(a_i) = \E_{a_{i'} \sim x_{i'}\^t :\ i' \neq i}[\ML_i(a_1, \ldots, a_{m})]$. Thus, for $1 \leq t \leq T$, 
  \begin{align}
    \left|\fds{h}{\ell_i}{t}(a_i)\right| =& \left|\sum_{s=0}^h {h \choose s} (-1)^{h-s} \ell_i\^{t+s}(a_i) \right|\label{eq:use-bin-coeffs-expan} \\
    =& \left|\sum_{a_{i'} \in [n_{i'}],\ \forall i' \neq i} \ML_i(a_1, \ldots, a_m) \sum_{s=0}^h {h \choose s} (-1)^{h-s}  \cdot \prod_{i' \neq i} x_{i'}\^{t+s}(a_{i'})\right| \nonumber\\
    \leq  & \sum_{a_{i'} \in [n_{i'}], \ \forall i' \neq i} \left| \sum_{s=0}^h {h \choose s} (-1)^{h-s} \cdot \prod_{i' \neq i} x_{i'}\^{t+s}(a_{i'})  \right| \nonumber\\
    =& \sum_{a_{i'} \in [n_{i'}], \ \forall i' \neq i} \left| \fds{h}{\left( \prod_{i' \neq i} x_{i'}(a_{i'}) \right)}{t} \right|,\label{eq:xi-seq}
  \end{align}
  where (\ref{eq:use-bin-coeffs-expan}) and (\ref{eq:xi-seq}) use Remark \ref{rem:fds-alt} and in (\ref{eq:xi-seq}), $\prod_{i'\neq i} x_{i'}(a_{i'})$ refers to the sequence $\prod_{i' \neq i} x_{i'}\^1(a_{i'}),$ $\prod_{i' \neq i} x_{i'}\^2(a_{i'}),\ldots$,  $\prod_{i' \neq i} x_{i'}\^T(a_{i'})$.

  In the remainder of this lemma we will prepend to the loss sequence $\ell_i\^1, \ldots, \ell_i\^T$ the vectors $\ell_i\^0 = \ell_i\^{-1} := \bbzero \in \BR^{n_i}$. We will also prepend $x_i\^0 := x_i\^1 =  (1/n_i, \ldots, 1/n_i) \in \Delta^{n_i}$ to the strategy sequence $x_i\^1, \ldots, x_i\^T$.  
  Next notice that for any agent $i \in [m]$, any $t_0 \in \{0, 1, \ldots, T \}$, and any $t \geq 0$, by the definition (\ref{eq:opt-hedge}) of the \Opthedge updates, it holds that, for each $j \in [n_i]$, 
  $$
x_i\^{t_0+t+1}(j) = \frac{x_i\^{t_0}(j) \cdot \exp\left( \eta \cdot \left( \ell_i\^{t_0-1}(j) - \sum_{s=0}^t \ell_i\^{t_0+s}(j) - \ell_i\^{t_0+t}(j)\right) \right)}{\sum_{k=1}^{n_i} x_i\^{t_0}(k) \cdot \exp\left( \eta \cdot \left( \ell_i\^{t_0-1}(k) - \sum_{s=0}^t \ell_i\^{t_0+s}(k) - \ell_i\^{t_0+t}(k)\right) \right)}.
$$
Note in particular that our definitions of $\ell_i\^0, \ell_i\^{-1}, x_i\^0$ ensure that the above equation holds even for $t_0 \in \{0,1\}$. \noah{check this, also $\phi$ should depend on $i$} 
Now an integer $t_0$ satisfying $0 \leq t_0 \leq T$; for $t \geq 0$, let us write 
$$
\bar \ell_{i,t_0}\^{t} := \ell_i\^{t_0-1} - \sum_{s=0}^{t-1} \ell_i\^{t_0+s} - \ell_i\^{t_0+t-1}.
$$
Also, for a vector $z = (z(1), \ldots, z(n_i)) \in \BR^{n_i}$ and an index $j \in [n_i]$, define
\begin{equation}
  \label{eq:def-phij}
\phi_{t_0,j}(z) := \frac{ \exp\left(  z(j)\right)}{\sum_{k=1}^{n_i} x_i\^{t_0}(k) \cdot \exp\left( z(k)\right)},
\end{equation}
so that $x_i\^{t_0+t}(j) = x_i\^{t_0}(j)\cdot \phi_{t_0,j}(\eta \cdot \bar \ell_{i,t_0}\^{t})$ for $t \geq 1$. In particular, for any $i \in [m]$, and any choices of $a_{i'} \in [n_{i'}]$ for all $i' \neq i$,
\begin{align}
\prod_{i' \neq i} x_{i'}\^{t_0+t}(a_{i'}) = \prod_{i' \neq i} x_{i'}\^{t_0}(a_{i'}) \cdot \phi_{t_0,a_{i'}}(\eta \cdot \bar \ell_{i',t_0}\^{t})\label{eq:prod-aip-xit}.
\end{align}

Next, note that 
$$
\fds{1}{\bar \ell_{i,t_0}}{t} = \ell_i\^{t_0+t-1} - 2\ell_i\^{t_0+t} = \ell_i\^{t_0+t-1} - 2\shfs{1}{\ell_i}{t_0+t-1},
$$
meaning that for any $h' \geq 1$,
\begin{equation}
  \label{eq:barl-fds}
\fds{h'}{\bar \ell_{i,t_0}}{t} = \fds{h'-1}{\ell_i}{t_0+t-1} -2\shfs{1}{\fd{h'-1}{\ell_i}}{t_0+t-1}.
\end{equation}

We next establish the following claims which will allow us to prove Lemma \ref{lem:dh-bound} by induction.
\begin{claim}
  \label{clm:base-step}
For  any $t_0 \in \{0, 1, \ldots, T \}$, $t \geq 0$, and $i \in [m]$, it holds that $\| \bar \ell_{i,t_0}\^t \|_\infty \leq t+2$.
\end{claim}
\begin{proof}[Proof of Claim \ref{clm:base-step}]
The claim is immediate from the triangle inequality and the fact that $\| \ell_i\^t \|_\infty \leq 1$ for all $t \in [T]$.
\end{proof}

\begin{claim}
  \label{clm:inductive-step}
  Fix $h$ so that $1 \leq h \leq H$. Suppose that for some $B_0 \geq 3$ and for all $0 \leq h' < h$, all $i \in [m]$, and all $t \leq T-h'$, it holds that $\| \fds{h'}{\ell_i}{t} \|_\infty \leq    \alpha^{h'} \cdot (h'+1)^{B_0 (h'+1) }$. Suppose that the step size $\eta$ satisfies $\eta \leq \min \left\{ \frac{\alpha}{36 e^5 m}, \frac{1}{12e^5 (H+3)m}\right\}$.  
  Then for all $i \in [m]$ and $1 \leq t \leq T-h$, 
  \begin{equation}
    \label{eq:inductive-game-formal}
\left\| \fds{h}{\ell_i}{t}  \right\|_\infty \leq  \alpha^h \cdot h^{B_0 h + 1}. 
  \end{equation}
\end{claim}
\begin{proof}[Proof of Claim \ref{clm:inductive-step}]
Set $B_1 := 12 e^5 m$, so that the assumption of the claim gives $\eta \leq \min \left\{\frac{\alpha}{3 B_1}, \frac{1}{B_1(H+3)} \right\}$. 
  
We first use Lemma \ref{lem:fd-analytic} to bound, for each $0 \leq t_0 \leq T-h$, $i \in [m]$, and $a_{i'} \in [n_{i'}]$ for all $i' \neq i$, the quantity $\left| \fds{h}{\left( \prod_{i' \neq i} x_{i'}(a_{i'}) \right)}{t_0+1} \right|$ . In particular, we will apply Lemma \ref{lem:fd-analytic} with $n=\sum_{i' \neq i} n_{i'}$, $\nu=1$, the value of $h$ in the statement of Claim \ref{clm:inductive-step}, $T = h+1$, and the sequence $L\^t$, for $1 \leq t \leq h+1$, defined as
$$
L\^t = \left( \eta \cdot \bar \ell_{1,t_0}\^t, \ldots, \eta \cdot \bar \ell_{i-1,t_0}\^t, \eta \cdot \bar \ell_{i+1,t_0}\^t, \ldots, \eta \cdot \bar \ell_{m,t_0}\^t \right),
$$
namely the concatenation of the vectors $\eta \cdot \bar \ell_{1,t_0}\^t, \ldots, \eta \cdot \bar \ell_{i-1,t_0}\^t, \eta \cdot \bar \ell_{i+1,t_0}\^t, \ldots, \eta \cdot \bar \ell_{m,t_0}\^t$. 
The function $\phi$ in Lemma \ref{lem:fd-analytic} is set to the function that takes as input the concatenation of $z_{i'} \in \BR^{n_{i'}}$ for all $i' \neq i$ and outputs:
\begin{align}
\phi_{t_0, a_{-i}}(z_1, \ldots, z_{i-1}, z_{i+1}, \ldots, z_m) := \prod_{i' \neq i} \phi_{t_0,a_{i'}}(z_{i'})\label{eq:def-phimi},
\end{align}
where the function $\phi_{t_0,a_{i'}}$ are as defined in (\ref{eq:def-phij}). 
We first verify the preconditions of Lemma \ref{lem:fd-analytic}. By Lemma \ref{lem:prodmax-ak-bound}, $\phi_{t_0, a_{-i}}$ is a $(1, e^3 m)$-bounded function for some constant $C \geq 1$. 
By Lemma \ref{lem:softmax-roc}, the radius of convergence of each function $\phi_{t_0, a_{i'}}$ at $\bbzero$ is at least 1; thus the radius of convergence of $\phi_{t_0, a_{-i}}$ at $\bbzero$ is at least $\nu=1$. 
  Claim \ref{clm:base-step} gives that $\| \bar \ell_{i,t_0}\^t \|_\infty \leq t + 2 \leq h+3$ for all $t \leq h+1$. Thus, since $\eta \leq \frac{1}{B_1(H+3)}$,
  \begin{align}
\left\| \fds{0}{\left(\eta \cdot \bar \ell_{i,t_0}\right)}{t} \right\|_\infty =    \| \eta \cdot \bar \ell_{i,t_0}\^t \|_\infty \leq \eta \cdot (H+3) \leq \frac{1}{B_1} \nonumber
  \end{align}
for $1 \leq t \leq h_0+1$. Next, for $1 \leq h' \leq h$ and $1 \leq t \leq h+1-h'$, we have
  \begin{align}
    \left\| \fds{h'}{(\eta \cdot \bar \ell_{i,t_0})}{t} \right\|_\infty \leq & \eta \cdot  \left\| \fds{h'-1}{\ell_i}{t_0+t-1} \right\|_\infty + 2\eta \cdot \left\| \fds{h'-1}{\ell_i}{t_0 + t} \right\|_\infty \label{eq:3triangle} \\
    \label{eq:use-ind-l}    \leq & 3 \eta \cdot \alpha^{h'-1} \cdot (h')^{B_0 (h')}\\
    \leq & 
           \frac{1}{B_1} \cdot \alpha^{h'} \cdot (h')^{B_0 (h')},\label{eq:hm1-to-h}
  \end{align}
  where (\ref{eq:3triangle}) follows from (\ref{eq:barl-fds}), (\ref{eq:use-ind-l}) follows from the assumption in the statement of Claim \ref{clm:inductive-step} and $t_0  + t + h'-1 \leq t_0 + h \leq T$, and (\ref{eq:hm1-to-h}) follows from the fact that $3\eta \leq \frac{\alpha}{B_1}$. It then follows from Lemma \ref{lem:fd-analytic} and (\ref{eq:prod-aip-xit}) that 
  \begin{align}
    & \frac{1}{\prod_{i' \neq i} x_{i'}\^{t_0}(a_{i'})} \cdot \left|\fds{h}{\left(\prod_{i' \neq i} {x_{i'}(a_{i'})}\right)}{t_0+1} \right| \nonumber\\
    =& \left|\fds{h}{\left(\prod_{i' \neq i} \left( \phi_{t_0,a_{i'}} \circ \eta  \bar \ell_{i',t_0}\right)\right)}{t_0+1} \right| \label{eq:fractional-decompose-prods}\\
    = & \left| \fds{h}{\left( \phi_{t_0,a_{-i}} \circ (\eta \bar \ell_{1,t_0}, \ldots, \eta\bar \ell_{i-1,t_0}, \eta \bar \ell_{i+1,t_0}, \ldots, \eta \bar \ell_{m,t_0})  \right)}{1} \right| \label{eq:use-phi-defn}\\
    \leq & \frac{12 e^5m}{B_1} \cdot \alpha^{h} \cdot h^{B_0 h +1} 
           = \alpha^{h} \cdot (h)^{B_0h+1}.\label{eq:bound-dh-xis}
  \end{align}
  (In particular, (\ref{eq:fractional-decompose-prods}) uses (\ref{eq:prod-aip-xit}), (\ref{eq:use-phi-defn}) uses the definition of $\phi_{t_0,a_{-i}}$ in (\ref{eq:def-phimi}), and (\ref{eq:bound-dh-xis}) uses Lemma \ref{lem:fd-analytic}.)

  Next we use (\ref{eq:xi-seq}), which gives that for each $i \in [m]$ and $t \geq 1$,\noah{check that this actually works for $t = 1$}
  \begin{align}
    \left \| \fds{h}{\ell_i}{t}  \right\|_\infty \leq & \sum_{a_{i'} \in [n_{i'}], \ \forall i' \neq i} \left| \fds{h}{\left( \prod_{i' \neq i} x_{i'}(a_{i'}) \right)}{t} \right| \nonumber\\ 
    \leq & \sum_{a_{i'} \in [n_{i'}],\ \forall i' \neq i} \prod_{i' \neq i} x_{i'}\^{t_0}(a_{i'}) \cdot \alpha^{h} \cdot (h)^{B_0 h+1}  \label{eq:use-l1-bound}\\
    = &\alpha^h \cdot (h)^{B_0 h+1}\nonumber, 
  \end{align}
  where (\ref{eq:use-l1-bound}) follows from (\ref{eq:bound-dh-xis}) with $t = t_0 - 1$ (here we use that $t_0$ may be 0). 
  This completes the proof of Claim \ref{clm:inductive-step}.
\end{proof}

It is immediate that for all $i \in [m], t \in [T]$, we have that $\| \fds{0}{\ell_i}{t} \|_\infty \leq 1 = \alpha^0 \cdot 1^{B_0 \cdot 1}$. We now apply Claim \ref{clm:inductive-step} inductively with $B_0 = 3$, for which it suffices to have $\eta \leq \frac{\alpha}{36e^5 m}$ as long as $\alpha \leq 1/(H+3)$. This gives that for $0 \leq h \leq H$, $i \in [m]$, and $t \in [T-h]$, $\| \fds{h}{\ell_i}{t} \|_\infty \leq \alpha^h \cdot h^{3h+1}$, completing the proof of Lemma \ref{lem:dh-bound}. \noah{todo check i didn't miss anything}
\end{proof}

\section{Proofs for Section \ref{sec:downwards-ind}}
The main goal of this section is to prove Lemma \ref{lem:d210}. First, in Section \ref{sec:d210-prelim} we prove some preliminary lemmas and then we prove Lemma \ref{lem:d210} in Section \ref{sec:d210-proof}

\subsection{Preliminary lemmas}
\label{sec:d210-prelim}
Lemma \ref{lem:covvar-close} shows that $\Var{P}{W}$ and $\Var{P'}{W}$ are close when the entries of $P,P'$ are close; it will be applied with $P,P'$ equal to the strategies $x_i\^t \in \Delta^{n_i}$ played in the course of \Opthedge.
\begin{lemma}
  \label{lem:covvar-close}
  Suppose $n\in \BN$ and $M > 0$ are given, and $W \in \BR^n$ is a vector. Suppose $P,P' \in \Delta^n$ are distributions with $\max\left\{\left\| \frac{P}{P'} \right\|_\infty, \left\| \frac{P'}{P} \right\|_\infty \right\} \leq 1+\alpha$ for some $\alpha > 0$. Then
  \begin{align}
    (1-\alpha) \Var{P}{W} \leq \Var{P'}{W} \leq (1 +  \alpha) \Var{P}{W}. \label{eq:var-close}
  \end{align}
\end{lemma}
\begin{proof}
We first prove that $\Var{P'}{W} \leq (1+\alpha) \Var{P}{W}$. To do so, note that since adding a constant to every entry of $W$ does not change $\Var{P}{W}$ or $\Var{P'}{W}$, by replacing $W$ with $W - \lng P, W \rng \cdot \bbone$, we may assume without loss of generality that $\lng P, W \rng = 0$. Thus $\Var{P}{W} = \sum_{j=1}^n P(j) W(j)^2$. 
  Now we may compute:
  \begin{align}
    \Var{P'}{W} 
    \leq & \sum_j P'(j) \cdot W(j)^2\nonumber \\
    = & \sum_j P(j) \cdot W(j)^2 + \sum_j (P'(j) - P(j)) \cdot W(j)^2  \nonumber\\
    =&  (1+\alpha) \Var{P}{W}\label{eq:use-pprime-alpha},
  \end{align}
  where (\ref{eq:use-pprime-alpha}) uses the fact that $\left\| \frac{P'}{P} \right\|_\infty \leq 1+\alpha$.

  By interchanging the roles of $P,P'$, we obtain that
  \begin{align}
\Var{P'}{W} \geq \frac{1}{1+\alpha} \Var{P}{W} \geq (1-\alpha) \Var{P}{W}\nonumber.
  \end{align}
  This completes the proof of the lemma.
\end{proof}

Next we prove Lemma \ref{lem:freq-cauchy} (recall that only the special case $\mu = 0$ was proved in Section \ref{sec:downwards-ind}). For convenience the lemma is repeated below.
\begin{replemma}{lem:freq-cauchy}[Restated]
  Suppose $\mu \in \BR$, $\alpha > 0$, and $W\^0, \ldots, W\^{S-1} \in \BR$ is a sequence of reals satisfying
  \begin{equation}
    \label{eq:d2d1-appendix}
\sum_{t=0}^{S-1} \left( \fdcs{2}{W}{t} \right)^2 \leq \alpha \cdot \sum_{t=0}^{S-1} \left( \fdcs{1}{W}{t} \right)^2 + \mu.
\end{equation}
Then
$$
\sum_{t=0}^{S-1} \left( \fdcs{1}{W}{t} \right)^2 \leq \alpha \cdot \sum_{t=1}^{S-1} (W\^t)^2 + \mu/\alpha.
$$
\end{replemma}
To prove Lemma \ref{lem:freq-cauchy} we need the following basic facts about the Fourier transform:
\begin{fact}[Parseval's equality]
  \label{fac:parseval}
  It holds that $\sum_{t=0}^{S-1} |W\^t|^2 = \frac{1}{S}  \sum_{s=0}^{S-1} |\wh{W}\^s|^2$.
\end{fact}
The second fact gives a formula for the Fourier transform of the circular finite differences; its simple form is the reason we work with \emph{circular} finite differences in this section:
\begin{fact}
  \label{fac:fourier-circ}
For $h \in \BZ_{\geq 0}$, $\wh{\fdc{h}{W}}\^s = \wh{W}\^s \cdot (e^{2\pi i st/S} - 1)^h$. 
\end{fact}

\begin{proof}[Proof of Lemma \ref{lem:freq-cauchy}]
Note that the discrete Fourier transform of $\fdc{1}{W}$ satisfies $\widehat{\fdc{1}{W}}\^s = \widehat{W}\^s \cdot (e^{2\pi i s / T} - 1)$, and similarly $\widehat{\fdc{2}{W}}\^s = \widehat{W}\^s \cdot (e^{2\pi i s/T} - 1)^2$, for $0 \leq s \leq S-1$. 
  By the Cauchy-Schwarz inequality, Parseval's equality (Fact \ref{fac:parseval}), Fact \ref{fac:fourier-circ}, and the assumption that (\ref{eq:d2d1-appendix}) holds, we have
  \begin{align}
  \sum_{t=0}^{S-1} \left( \fdcs{1}{W}{t} \right)^2  =& \frac{1}{S} \sum_{s=0}^{S-1} \left| \widehat{\fdc{1}{W}}\^s \right|^2\nonumber\\
    =&\frac{1}{S} \sum_{s=0}^{S-1} \left| \widehat{W}\^s \cdot (e^{2\pi i s /T - 1}) \right|^2 \nonumber\\
    =&\frac{1}{S}\sum_{s=0}^{S-1} \left| \widehat{W}\^s \right| \cdot \left| \widehat{W}\^s \right| \left| e^{2\pi i s /T} - 1\right|^2 \nonumber\\
                                                 \leq &\sqrt{\frac{1}{S}\sum_{s=0}^{S-1} \left| \widehat{W}\^s \right|^2} \cdot \sqrt{\frac{1}{S}\sum_{s=0}^{S-1} \left| \widehat{W}\^s \right|^2 \cdot \left| e^{2\pi i s/T} - 1\right|^4}\nonumber\\
    =& \sqrt{ \sum_{t=0}^{S-1} (W\^t)^2} \cdot \sqrt{\frac{1}{S}\sum_{s=0}^{S-1} \left| \widehat{\fdc{2}{W}}\^s \right|^2}\nonumber\\
    =& \sqrt{ \sum_{t=0}^{S-1} (W\^t)^2} \cdot \sqrt{\sum_{t=0}^{S-1} \left( \fdcs{2}{W}{t}\right)^2}\nonumber\\
    \leq & \sqrt{ \sum_{t=0}^{S-1} (W\^t)^2} \cdot \sqrt{\alpha \cdot \sum_{t=0}^{S-1} \left( \fdcs{1}{W}{t} \right)^2 + \mu}\label{eq:cauchy-parseval}.
  \end{align}
  Note that for real numbers $A > 0$ and $\ep$ with $A + \ep > 0$, it holds that
  $$
\frac{A^2}{{A + \ep}} = \frac{{A}}{1 + {\ep/A}} \geq A \cdot (1 - \ep/A) = A - \ep.
$$
Taking $A = \sum_{t=0}^{S-1} \left( \fdcs{1}{W}{t} \right)^2$ and $\ep = \mu/\alpha$ (for which $A + \ep > 0$ is immediate) and using (\ref{eq:cauchy-parseval}) then gives
$$
\sum_{t=0}^{S-1} \left( \fdcs{1}{W}{t} \right)^2 - \mu/\alpha \leq \frac{\left( \sum_{t=0}^{S-1} \left( \fdcs{1}{W}{t} \right)^2 \right)^2}{{ \sum_{t=0}^{S-1} \left( \fdcs{1}{W}{t} \right)^2 + \mu/\alpha}}\leq \alpha \cdot \sum_{t=0}^{S-1} \left( W\^t \right)^2,
$$
  as desired.
\end{proof}

\subsection{Proof of Lemma \ref{lem:d210}}
\label{sec:d210-proof}
Now we prove Lemma \ref{lem:d210}. For convenience we restate the lemma below with the exact value of the constant $C_0$ referred to in the version in Section \ref{sec:downwards-ind}.
\begin{replemma}{lem:d210}[Restated]
  For any $M, \zeta, \alpha > 0$ and $n \in \BN$, suppose that $P\^1, \ldots, P\^T \in \Delta^n$ and $Z\^1, \ldots, Z\^T \in [-M,M]^n$ satisfy the following conditions:
  \begin{enumerate}
  \item \label{it:prec-consec-close-body-apx} The sequence $P\^1, \ldots, P\^T$ is $\zeta$-consecutively close for some $\zeta \in [1/(2T), \alpha^4/8256]$. 
  \item \label{it:prec-vars-body-apx} It holds that 
$ 
      \sum_{t=1}^{T-2} \Var{P\^t}{\fds{2}{Z}{t}} \leq \alpha \cdot \sum_{t=1}^{T-1} \Var{P\^t}{\fds{1}{Z}{t}} + \mu.
$ 
\end{enumerate}
Then
\begin{equation}
    \sum_{t=1}^{T-1} \Var{P\^t}{\fds{1}{Z}{t}} \leq \alpha \cdot (1+ \alpha) \sum_{t=1}^T \Var{P\^t}{Z\^t} + \frac{\mu}{\alpha} + \frac{1290 M^2}{\alpha^3} \label{eq:d210-conclusion}.
  \end{equation}
\end{replemma}
\begin{proof}
  Fix a positive integer $S < 1/(2\zeta) < T$, to be specified exactly below. For $1 \leq t_0 \leq T-S+1$, define $\mu_{t_0} \in \BR$ by 
  \begin{equation}
    \label{eq:def-mut0}
\mu_{t_0} =   \sum_{s=0}^{S-3} \Var{P\^{t_0+s}}{\fds{2}{Z}{t_0+s}} - \alpha\cdot \sum_{s=0}^{S-3} \Var{P\^{t_0+s}}{\fds{1}{Z}{t_0+s}}.
  \end{equation}
  Then 
  \begin{align}
    &\sum_{t_0 = 1}^{T-S+1} \mu_{t_0} \nonumber\\
    =& \sum_{t=1}^{T-2} \Var{P\^t}{\fds{2}{Z}{t}} \cdot \min \{ S-2, t, T-t-1 \} \nonumber\\
    &- \alpha \cdot \sum_{t=1}^{T-1} \Var{P\^t}{\fds{1}{Z}{t}} \cdot \min\{ S-2, t, T-t-1\}\nonumber\\
    \leq&(S-2) \cdot \sum_{t=1}^{T-2} \Var{P\^t}{\fds{2}{Z}{t}} - (S-2) \alpha \cdot \sum_{t=1}^{T-1} \Var{P\^t}{\fds{1}{Z}{t}} + 8\alpha (S-2)^2 M^2 \label{eq:use-d1z-bound}\\
    \leq& (S-2) \mu + 2\alpha(S-2)^2 M^2\label{eq:bound-sum-mut0},
  \end{align}
  where (\ref{eq:use-d1z-bound}) uses the fact that $\| Z\^t \|_\infty \leq M$ and so $\| \fds{1}{Z}{t} \|_\infty \leq 2M$ for all $t \in [T]$, and the final inequality (\ref{eq:bound-sum-mut0}) follows from assumption \ref{it:prec-vars-body-apx} of the lemma statement. 

  By (\ref{eq:def-mut0}) and Lemma \ref{lem:covvar-close} with $P = P\^{t_0}$, we have, for some constant $C > 0$,
  \begin{align}
    \sum_{s=0}^{S-3} \Var{P\^{t_0}}{\fds{2}{Z}{t_0+s}} \leq & (1 + 2\zeta  S) \cdot \sum_{s=0}^{S-3} \Var{P\^{t_0+s}}{\fds{2}{Z}{t_0+s}} \nonumber\\
    =& (1+ 2\zeta  S)\alpha\cdot \sum_{s=0}^{S-3} \Var{P\^{t_0+s}}{\fds{1}{Z}{t_0+s}} + (1 + 2\zeta  S) \mu_{t_0}\nonumber\\
    \leq & (1 + 2\zeta S)^2 \alpha \cdot \sum_{s=0}^{S-3} \Var{P\^{t_0}}{\fds{1}{Z}{t_0+s}} + (1 + 2\zeta S) \mu_{t_0}\label{eq:s21-mu}.
  \end{align}
  Here we have used that for $0 \leq s \leq S$, it holds that $\max \left\{ \left\|\frac{P\^{t_0+s}}{P\^{t_0}}\right\|_\infty, \left\|\frac{P\^{t_0}}{P\^{t_0+s}}\right\|_\infty \right\} \leq (1 + \zeta)^{S} \leq 1 + 2 \zeta S$ since $\zeta S \leq 1/2$.

  For any integer $1 \leq t_0 \leq T-S+1$, we define the sequence $Z_{t_0}\^{s} := Z\^{t_0+s}- \lng{P\^{t_0}},Z\^{t_0+s}\rng\bbone$, for $0 \leq s \leq S-1$. 
Thus $\lng Z_{t_0}\^s, P\^{t_0} \rng = 0$ for $0 \leq s \leq S-1$, which implies that for all $h \geq 0$, $0 \leq s \leq S-1$, $\lng \fdcs{h}{Z_{t_0}}{s}, P\^{t_0} \rng = 0$, and thus
  \begin{equation}
    \label{eq:var-zt0}
    \Var{P\^{t_0}}{\fdcs{h}{Z_{t_0}}{s}} = \sum_{j=1}^n P\^{t_0}(j) \cdot \fdcs{h}{Z_{t_0}}{s}(j)^2.
  \end{equation}
By the definition of the sequence $Z_{t_0}$, for $0 \leq s \leq S-h-1$, we have
  \begin{equation}
    \label{eq:var-zt0-2}
    \Var{P\^{t_0}}{\fds{h}{Z}{t_0+s}} = \Var{P\^{t_0}}{\fds{h}{Z_{t_0}}{s}} = \Var{P\^{t_0}}{\fdcs{h}{Z_{t_0}}{s}}.
  \end{equation}
  
For $1 \leq t_0 \leq T-S+1$, let us now define
  \begin{align}
    \nu_{t_0, j} := \sum_{s=0}^{S-1} \fdcs{2}{Z_{t_0}}{s}(j)^2 - (1 + 2\zeta S)^2 \alpha \cdot \sum_{s=0}^{S-1} \fdcs{1}{Z_{t_0}}{s}(j)^2 ,
    \label{eq:def-nutj}
  \end{align}
  so that, by (\ref{eq:s21-mu}), (\ref{eq:var-zt0}), and (\ref{eq:var-zt0-2}),
  \begin{align}
    & \sum_{j=1}^n P\^{t_0}(j) \cdot \nu_{t_0,j}\nonumber\\
    =&\sum_{s=0}^{S-1} \Var{P\^{t_0}}{\fdcs{2}{Z_{t_0}}{s}} - (1 + 2\zeta S)^2 \alpha \cdot \sum_{s=0}^{S-1} \Var{P\^{t_0}}{\fdcs{1}{Z_{t_0}}{s}} \nonumber \\
    \leq & \left(\sum_{s=0}^{S-3} \Var{P\^{t_0}}{\fds{2}{Z}{t_0+s}}\right) +  \Var{P\^{t_0+1}}{\fdcs{2}{Z}{t_0+S-2}} + \Var{P\^{t_0+1}}{\fdcs{2}{Z}{t_0+S-1}}  \nonumber\\
    & - (1 + 2\zeta S)^2 \alpha \cdot \sum_{s=0}^{S-3} \Var{P\^{t_0}}{\fds{1}{Z}{t_0+s}} \nonumber\\
    \leq & (1 + 2\zeta S)\mu_{t_0} + \Var{P\^{t_0}}{\fdcs{2}{Z}{t_0+S-2}} + \Var{P\^{t_0}}{\fdcs{2}{Z}{t_0+S-1}}.    \label{eq:pnu-ub}
\end{align}
By (\ref{eq:def-nutj}) and Lemma \ref{lem:freq-cauchy} applied to the sequence $Z_{t_0}\^0, \ldots, Z_{t_0}\^{S-1}$, it holds that, for each $j \in [n]$, 
\begin{equation}
  \label{eq:1d-d1-ub}
\sum_{s=0}^{S-1} \fdcs{1}{Z_{t_0}}{s}(j)^2 \leq (1 + \zeta CS)^2 \alpha \cdot \sum_{s=0}^{S-1} Z_{t_0}\^s(j)^2 + \frac{\nu_{t_0,j}}{(1 + \zeta CS)^2 \alpha}.
\end{equation}
Then we have: 
\begin{align}
  &\sum_{s=0}^{S-2} \Var{P\^{t_0}}{\fds{1}{Z}{t_0+s}}\nonumber\\
  =& \sum_{s=0}^{S-2} \Var{P\^{t_0}}{\fdcs{1}{Z_{t_0}}{s}}\label{eq:pass-to-circ}\\
  \leq & (1 + 2\zeta S)^2 \alpha \cdot \sum_{s=0}^{S-1} \Var{P\^{t_0}}{Z_{t_0}\^s} + \sum_{j=1}^n P\^{t_0}(j) \cdot \frac{\nu_{t_0,j}}{(1+ 2\zeta S)^2 \alpha}\label{eq:circ-pass-to-d0}\\
  \leq & (1 + 2\zeta S)^2 \alpha \sum_{s=0}^{S-1} \Var{P\^{t_0}}{Z_{t_0}\^s} + \frac{\mu_{t_0}}{(1 + 2\zeta S) \alpha} + \frac{\Var{P\^{t_0}}{\fdcs{2}{Z}{t_0+S-2}} + \Var{P\^{t_0}}{\fdcs{2}{Z}{t_0+S-1}}}{(1 + 2\zeta S)^2 \alpha}  \label{eq:average-over-p},
\end{align}
where (\ref{eq:pass-to-circ}) follows from (\ref{eq:var-zt0-2}), (\ref{eq:circ-pass-to-d0}) follows from (\ref{eq:1d-d1-ub}) and (\ref{eq:var-zt0}), 
and (\ref{eq:average-over-p}) follows from (\ref{eq:pnu-ub}). Summing the above for $1 \leq t_0 \leq T - S+1$, we obtain, for some constant $C > 0$,
\begin{align}
  & (S-1) \cdot \sum_{t=1}^{T-1} \Var{P\^t}{\fds{1}{Z}{t}} \nonumber\\
  \leq & \sum_{t_0=1}^{T-S+1} \sum_{s=0}^{S-2} \Var{P\^{t_0+s}}{\fds{1}{Z}{t_0+s}} + 8 (S-1)^2 M^2 \label{eq:d1all-1}\\
  \leq & \sum_{t_0=1}^{T-S+1} (1 + 2\zeta S) \sum_{s=0}^{S-2} \Var{P\^{t_0}}{\fds{1}{Z}{t_0+s}} + 8 (S-1)^2 M^2 \label{eq:d1all-2}\\
  \leq & (1 + 2\zeta S)^3 \alpha \sum_{t_0=1}^{T-S+1}\sum_{s=0}^{S-1} \Var{P\^{t_0}}{Z_{t_0}\^s} + \sum_{t_0=1}^{T-S+1} \frac{\mu_{t_0}}{\alpha} + 8 (S-1)^2 M^2 \nonumber\\
  & + \sum_{t_0=1}^{T-S+1} \frac{\Var{P\^{t_0}}{\fdcs{2}{Z}{t_0+S-2}}+ \Var{P\^{t_0}}{\fdcs{2}{Z}{t_0+S-1}}}{(1 + 2\zeta S) \alpha}   \label{eq:d1all-4}\\
  \leq & (1 + 2\zeta S)^4 \alpha \sum_{t_0 = 1}^{T-S+1} \sum_{s=0}^{S-1} \Var{P\^{t_0+s}}{Z\^{t_0+s}} + \sum_{t_0=1}^{T-S+1} \frac{\mu_{t_0}}{\alpha} + 8 (S-1)^2 M^2 \nonumber\\
  &+ \frac{4}{(1+2\zeta S)\alpha } \sum_{t_0=1}^{T-S+1} \left[ \Var{P\^{t_0}}{Z\^{t_0+S-2}} + 3 \Var{P\^{t_0}}{Z\^{t_0+S-1}}\right.\nonumber\\
  & \left.+ 3 \Var{P\^{t_0}}{Z\^{t_0}} + \Var{P\^{t_0}}{Z\^{t_0+1}}\right]\label{eq:d1all-7}\\
  \leq & (1 + 2\zeta S)^4 \alpha S \sum_{t=1}^T \Var{P\^t}{Z\^t} + \sum_{t_0=1}^{T-S+1} \frac{\mu_{t_0}}{\alpha} + 8 (S-1)^2 M^2 \nonumber\\
  &+ \frac{32}{(1 + 2\zeta S) \alpha} \sum_{t=1}^T \Var{P\^t}{Z\^t} \label{eq:d1all-9}\\
  \leq & (1 + 2\zeta S)^4 \alpha S \cdot \left(1 + \frac{32}{\alpha^2 S} \right) \sum_{t=1}^T \Var{P\^t}{Z\^t} + \sum_{t_0=1}^{T-S+1} \frac{\mu_{t_0}}{\alpha} + 8 (S-1)^2 M^2 \label{eq:d1all-10}\\
  \leq & (1 + 2\zeta S)^4 \alpha S \cdot \left(1 + \frac{32}{\alpha^2 S} \right) \sum_{t=1}^T \Var{P\^t}{Z\^t} + \frac{(S-2)\mu}{\alpha}+ 10 (S-1)^2 M^2 \label{eq:d1all-11},
\end{align}
where:
\begin{itemize}
\item (\ref{eq:d1all-1}) follows since $\| Z\^t \|_\infty \leq M$ and thus $\| \fds{1}{Z}{t} \|_\infty \leq 2M$ for all $t$;
\item (\ref{eq:d1all-2}) follows from Lemma \ref{lem:covvar-close} and the fact that for $0 \leq s \leq S-2$, $\max \left\{ \left\|\frac{P\^{t_0+s}}{P\^{t_0}}\right\|_\infty, \left\|\frac{P\^{t_0}}{P\^{t_0+s}}\right\|_\infty \right\}  \leq 1 + 2 \zeta S$ as established above as a consequence of the fact that the distributions $P\^t$ are $\zeta$-consecutively close.
\item (\ref{eq:d1all-4}) follows from (\ref{eq:average-over-p});
\item The first term in (\ref{eq:d1all-7}) is bounded using Lemma \ref{lem:covvar-close} and the fact that the distributions $P\^t$ are $\zeta$-consecutively close, and the the final term in (\ref{eq:d1all-7}) is bounded using the fact that for any vectors $Z_1, \ldots, Z_k \in \BR^n$ and any $P \in \Delta^n$, we have $\Var{P}{Z_1 + \cdots + Z_k} \leq k\cdot \left( \Var{P}{Z_1} + \cdots + \Var{P}{Z_k} \right)$;
\item (\ref{eq:d1all-9}) and (\ref{eq:d1all-10}) by rearranging terms;
\item (\ref{eq:d1all-11}) follows from (\ref{eq:bound-sum-mut0}).
\end{itemize}
Now choose $S = \left\lceil \frac{128}{\alpha^3} \right\rceil$, so that $\frac{32}{\alpha^2 S} \leq \frac{{\alpha}}{4}$. Therefore, as long as $2\zeta S \leq \frac{\alpha}{32}$, we have, since $\alpha \leq 1/2$, that 
$$
(1 + 2\zeta S)^4 \alpha \cdot \frac{S}{S-1} \cdot \left(1 + \frac{32}{\alpha^2 S} \right) \leq \alpha \cdot  (1 + \alpha/4)^3 \leq \alpha \cdot (1 + \alpha).
$$
Then it follows from (\ref{eq:d1all-11}) that
\begin{align}
\sum_{t=1}^{T-1} \Var{P\^t}{\fds{1}{Z}{t}} \leq & \alpha (1+\alpha) \cdot \sum_{t=1}^T \Var{P\^t}{Z\^t} + \frac{\mu}{\alpha} + 10SM^2\label{eq:d1d0-final-done},
\end{align}
Using that $S \leq \frac{129}{\alpha^3}$, the inequality $2\zeta S \leq \alpha/32$ can be satisfied by ensuring that $\zeta \leq \frac{\alpha^4}{8256} =  \frac{\alpha^4}{129 \cdot 42}   \leq \frac{\alpha}{64S}$. Note that our choice of $S$ ensures that $\zeta S \leq 1/2$, as was assumed earlier. Moreover, 
we have $10SM^2 \leq \frac{1290 M^2}{\alpha^3}$. Thus, (\ref{eq:d1d0-final-done}) gives the desired result. 
\end{proof}

\subsection{Completing the proof of Theorem \ref{thm:polylog-main}}
\label{sec:polylog-main-completing-proof}
Using the lemma developed in the previous sections we now can complete the proof of Theorem \ref{thm:polylog-main}.  We begin by proving Lemma \ref{lem:bound-l-dl}. The lemma is restated formally below.
\begin{replemma}{lem:bound-l-dl}[Detailed]
  There are constants $C,C' > 1$ so that the following holds.  Suppose a time horizon $T \geq 4$ is given, we set $H := \lceil \log T \rceil$, and all players play according to \Opthedge with step size $\eta$ satisfying $1/T \leq \eta \leq \frac{1}{C \cdot mH^4}$. Then for any $i \in [m]$, the losses $\ell_i\^1, \ldots, \ell_i\^T \in [0,1]^{n_i}$ for player $i$ satisfy:
\begin{align}
  \label{eq:var-l-dl-apx}
\sum_{t=1}^T \Var{x_i\^t}{\ell_i\^t - \ell_i\^{t-1}} \leq \frac{1}{2} \cdot \sum_{t=1}^T \Var{x_i\^t}{\ell_i\^{t-1}} + C' H^5.
\end{align}
\end{replemma}

\noindent We state a generic version of this lemma that can be applied in more general settings.

\begin{lemma}\label{lem:general-bound-l-dl}
    For any integers $n \geq 2$ and $T \geq 4$, we set $H := \lceil \log T \rceil$, $\alpha = 1/(4H)$, and $\alpha_0 = \frac{\sqrt{\alpha/8}}{H^3}$.  Suppose that $Z\^1, \ldots, Z\^T \in [0,1]^{n}$ and $P\^1, \ldots, P\^T \in \Delta^{n}$ satisfy the following
    \begin{enumerate}
        \item \label{it:gen-down} For each $0 \leq h \leq H$ and $1 \leq t \leq T-h$, it holds that $\left\| \fds{h}{Z}{t} \right\|_\infty \leq H\cdot  \left(\alpha_0 H^{3}\right)^h$
        \item \label{it:gen-close} The sequence $P\^1, \ldots, P\^T$ is $\zeta$-consecutively close for some $\zeta \in [1/(2T), \alpha^4/8256]$. 
    \end{enumerate}
    Then,
    
    \begin{align}
    \sum_{t=1}^T \Var{P\^t}{Z\^t - Z\^{t-1}} \leq & 2\alpha \sum_{t=1}^T \Var{P\^t}{Z\^{t-1}} + 165120(1+\zeta) H^5 + 2
    \end{align}
\end{lemma}

First, we demonstrate how Lemma \ref{lem:general-bound-l-dl} implies Lemma \ref{lem:bound-l-dl}.
\begin{proof}[Proof of Lemma \ref{lem:bound-l-dl}]
We hope to apply Lemma \ref{lem:general-bound-l-dl} with $n=n_i$, $P\^t = x_i\^t$ and $Z\^t = \ell_i\^{t}$, as well as $\zeta = 7\eta$.  To do so, we must verify that the preconditions of Lemma \ref{lem:general-bound-l-dl} hold when our sequences arise from the dynamics of players playing \Opthedge with step size $\eta$ satisfying $1/T \leq \eta \leq \frac{1}{C \cdot mH^4}$.\\

Set $C_1 = 8256$ (note that $C_1$ is the constant appearing in item \ref{it:gen-close} of Lemma \ref{lem:general-bound-l-dl} and item \ref{it:prec-consec-close-body-apx} of Lemma \ref{lem:d210} in Section \ref{sec:d210-proof}).  Our assumption that $\eta \leq \frac{1}{C \cdot mH^4}$ implies that, as long as the constant $C$ satisfies $C \geq 4^4 \cdot 7 \cdot C_1 = 14794752$, 
\begin{equation}
  \label{eq:eta-reqs}
\eta \leq \min \left\{ \frac{\alpha^4}{7C_1}, \frac{\alpha_0}{36e^5 m}\right\}.
  \end{equation}

To verify precondition \ref{it:gen-down}, we apply Lemma \ref{lem:dh-bound} with the parameter $\alpha$ in the lemma set to $\alpha_0$: a valid selection as $\alpha_0 < \frac{1}{H+3}$.  We conclude that, for each $i \in [m]$, $0 \leq h \leq H$ and $1 \leq t \leq T-h$, it holds that $\left\| \fds{h}{\ell_i}{t} \right\|_\infty \leq H\cdot  \left(\alpha_0 H^{3}\right)^h$ since $\eta \leq \frac{\alpha_0}{36e^5m}$ as required by the lemma.  To verify precondition \ref{it:gen-close}, we first confirm that our selection of $\zeta = 7\eta$ places it in the desired interval $[1/(2T), \alpha^4/C_1]$ as $\eta \leq \frac{\alpha^4}{7C_1}$.  By the definition of the \Opthedge updates, for all $i \in [m]$ and $1 \leq t \leq T$, we have $\max \left\{\left\| \frac{x_i\^t}{x_i\^{t+1}}\right\|_\infty, \left\| \frac{x_i\^{t+1}}{x_i\^t}\right\|_\infty\right\} \leq \exp(6\eta)$.  Thus, the sequence $x_i\^1, \ldots, x_i\^T$ is $(7\eta)$-consecutively close (since $\exp(6\eta) \leq 1+7\eta$ for $\eta$ satisfying (\ref{eq:eta-reqs})).  Therefore, Lemma \ref{lem:general-bound-l-dl} applies and we have
\begin{align*}
    \sum_{t=1}^T \Var{x_i\^t}{\ell_i\^t - \ell_i\^{t-1}} &\leq 2\alpha \sum_{t=1}^T \Var{x_i\^t}{\ell_i\^{t-1}} + 165120(1+7\eta) H^5+2\\
    &\leq \frac{1}{2} \cdot \sum_{t=1}^T \Var{x_i\^t}{\ell_i\^{t-1}} + C' H^5
\end{align*}
for $C' = 2 + 165120(1+7/8256) = 165262$, as desired.
\end{proof}
So, it suffices to prove Lemma \ref{lem:general-bound-l-dl}.
\begin{proof}[Proof of Lemma \ref{lem:general-bound-l-dl}]
Set $\mu = H \cdot \left( \alpha_0 H^{3} \right)^H$. Therefore, from item \ref{it:gen-down} we have,
  \begin{align}
\sum_{t=1}^{T-H} \Var{P\^t}{\fds{H}{Z}{t}} \leq \alpha \cdot \sum_{t=1}^{T-H+1} \Var{P\^t}{\fds{H-1}{Z}{t}} + \mu^2 T. \label{eq:fds-H-bc}
  \end{align}
 We will now prove, via reverse induction on $h$, that for all $h$ satisfying $H-1 \geq h \geq 0$,
  \begin{equation}
    \label{eq:fds-H-ind}
\sum_{t=1}^{T-h-1} \Var{P\^t}{\fds{h+1}{Z}{t}} \leq \alpha \cdot (1 + 2\alpha)^{H-h-1} \cdot \sum_{t=1}^{T-h} \Var{P\^t}{\fds{h}{Z}{t}} +  \frac{2C_0 \cdot H^2  \cdot \left( {2}  \alpha_0 H^{3} \right)^{2h}}{\alpha^3}.
\end{equation}
with $C_0 = 1290$ (note that $C_0$ is the constant appearing in the inequality (\ref{eq:d210-conclusion}) of the statement of Lemma \ref{lem:d210} in Section \ref{sec:d210-proof}).  The base case $h = H-1$ is verified by (\ref{eq:fds-H-bc}) and the fact that $2^{2(H-1)} \geq 2^H \geq T$. 
Now suppose that (\ref{eq:fds-H-ind}) holds for some $h$ satisfying $H-1 \geq h \geq 1$. We will now apply Lemma \ref{lem:d210}, with $P\^t = P\^t$ and $Z\^t = \fds{h-1}{Z}{t}$ for $1 \leq t \leq T-h+1$, as well as $M = H \cdot \left( 2 \alpha_0 H^{3}\right)^{h-1}$, $\zeta = \zeta$, $\mu =  \frac{2C_0 \cdot H^2  \cdot \left( 2 \alpha_0 H^{3} \right)^{2h}}{\alpha^3}$, and the parameter $\alpha$ of Lemma \ref{lem:d210} set to $\alpha\cdot (1+2\alpha)^{H-h-1}$. We verify that precondition \ref{it:prec-consec-close-body-apx} holds due to precondition \ref{it:gen-close} of Lemma \ref{lem:general-bound-l-dl} and the fact that $\alpha \leq \alpha\cdot (1+2\alpha)^{H-h-1}$.  Moreover, precondition \ref{it:prec-vars-body-apx} holds by our inductive hypothesis (\ref{eq:fds-H-ind}) and our choice of $\mu$.  Therefore, by Lemma \ref{lem:d210} and the fact that $1 + \alpha \cdot (1 + 2 \alpha)^H \leq 1 + 2\alpha$ for our choice of $\alpha = 1/(4H)$, it follows that
\begin{align}
  \sum_{t=1}^{T-h} \Var{P\^t}{\fds{h}{Z}{t}} \leq & \alpha \cdot (1 + 2\alpha)^{H-h} \cdot \sum_{t=1}^{T-h+1} \Var{P\^t}{\fds{h-1}{Z}{t}} + \frac{2C_0 \cdot H^2 \cdot \left(2 \alpha_0 H^{3} \right)^{2h}}{\alpha^4}  \nonumber\\
                                                    &  + \frac{C_0 \cdot H^2 \cdot \left( 2\alpha_0 H^{3} \right)^{2(h-1)}}{\alpha^3}\nonumber\\
  \leq & \alpha \cdot (1 + 2\alpha)^{H-h} \cdot \sum_{t=1}^{T-h+1} \Var{P\^t}{\fds{h-1}{Z}{t}}\nonumber\\
                                                    & + \frac{C_0 \cdot H^2 \cdot (2\alpha_0 H^3)^{2(h-1)}}{\alpha^3} \cdot \left(1 + \frac{2 (2\alpha_0 H^3)^2}{\alpha} \right)  \nonumber\\
  \leq &  \alpha \cdot (1 + 2\alpha)^{H-h} \cdot \sum_{t=1}^{T-h+1} \Var{P\^t}{\fds{h-1}{Z}{t}} +\frac{2C_0 \cdot H^2 \cdot (2\alpha_0 H^3)^{2(h-1)}}{\alpha^3} \nonumber,
\end{align}
where the final equality follows since $\alpha_0$ is chosen so that $2(2\alpha_0H^3)^2=\alpha$. 
This completes the proof of the inductive step. Thus (\ref{eq:fds-H-ind}) holds for $h=0$. Using again that the sequence $P\^t$ is $(\zeta)$-exponentially close, we see that 
\begin{align}
  & \sum_{t=1}^T \Var{P\^t}{Z\^t - Z\^{t-1}}\nonumber\\
  \leq & 1 + \sum_{t=2}^T \Var{P\^t}{Z\^t - Z\^{t-1}} \nonumber\\
  \leq & 1 + (1+ \zeta) \sum_{t=1}^{T-1} \Var{P\^t}{\fds{1}{Z}{t}} \label{eq:use-xi-close-1}\\
  \label{eq:use-d1d0-inequality}
  \leq & 1 + (1 + \zeta) \cdot \left( \alpha (1 + 2\alpha)^{H-1} \sum_{t=1}^T \Var{P\^t}{Z\^t} + \frac{2 C_0 H^2}{\alpha^3} \right)\\
  \leq & 2 + (1 + \zeta) \cdot \left( \alpha(1+2\alpha)^{H-1} (1 + \zeta) \sum_{t=2}^T \Var{P\^t}{Z\^{t-1}} + \frac{2 C_0 H^2}{\alpha^3} \right)\label{eq:use-xi-close-2}\\
  \leq & \alpha (1 + 2\alpha)^H \sum_{t=1}^T \Var{P\^t}{Z\^{t-1}} + \frac{2(1+\zeta) C_0 H^2}{\alpha^3} + 2\nonumber\\
  \leq & 2\alpha \sum_{t=1}^T \Var{P\^t}{Z\^{t-1}} + \frac{2(1+\zeta) C_0 H^2}{\alpha^3} + 2\label{eq:final-210-statement},
\end{align}
where (\ref{eq:use-xi-close-1}) and (\ref{eq:use-xi-close-2}) follow from Lemma \ref{lem:covvar-close}, and (\ref{eq:use-d1d0-inequality}) uses (\ref{eq:fds-H-ind}) for $h=0$. 
Now, (\ref{eq:final-210-statement}) verifies the statement of the lemma as $C_0=1290$ and $\alpha = \frac{1}{4H}$.
\end{proof}

We are finally ready to prove Theorem \ref{thm:polylog-main}. For convenience the theorem is restated below.
\begin{reptheorem}{thm:polylog-main}[Restated]
  There are constants $C, C' >1$ so that the following holds. Suppose a time horizon $T \in \BN$ is given. Suppose all players play according to \Opthedge with any positive step size $\eta \leq \frac{1}{C \cdot m \log^4 T}$. Then for any $i \in [m]$, the regret of player $i$ satisfies
  \begin{align}
\Reg{i}{T} \leq \frac{\log n_i}{\eta} + C' \cdot \log T \label{eq:reg-eta-ub-apx}.
  \end{align}
  In particular, if the players' step size is chosen as $\eta = \frac{1}{C \cdot m \log^4 T}$, then the regret of player $i$ satisfies
  \begin{align}
\Reg{i}{T} \leq O \left( m \cdot \log n_i \cdot \log^4 T \right).\label{eq:reg-eta-fix-apx}
  \end{align}
\end{reptheorem}
\begin{proof}
  The conclusion of the theorem is immediate if $T < 4$, so we may assume from here on that $T \geq 4$. Moreover, the conclusion of (\ref{eq:reg-eta-ub-apx}) is immediate if $\eta \leq 1/T$ (as $\Reg{i}{T} \leq T$ necessarily), so we may also assume that $\eta \geq 1/T$. Let $C''$ be the constant $C$ of Lemma \ref{lem:omwu-local}, let $B$ be the constant called $C$ in Lemma \ref{lem:bound-l-dl} and $B'$ be the constant called $C'$ in Lemma \ref{lem:bound-l-dl}. As long as the constant $C$ of Theorem \ref{thm:polylog-main} is chosen so that $C \geq B$ and $\eta \leq \frac{1}{C \cdot m \log^4 T}$ implies that $C'' \eta \leq 1/6$, we have the following: 
  \begin{align}
    \Reg{i}{T} \leq & \frac{\log n_i}{\eta} + \sum_{t=1}^T \left(\frac{\eta}{2} + C \eta^2\right)  \Var{x_i\^t}{\ell_i\^t - \ell_i\^{t-1}} - \sum_{t=1}^T \frac{(1-C\eta) \eta}{2}  \Var{x_i\^{t}}{\ell_i\^{t-1}}\label{eq:use-omwu-local}\\
    \leq & \frac{\log n_i}{\eta} + \frac{2\eta}{3} \sum_{t=1}^T \Var{x_i\^t}{\ell_i\^t - \ell_i\^{t-1}} -\frac{\eta}{3}  \sum_{t=1}^T \Var{x_i\^{t}}{\ell_i\^{t-1}}\nonumber\\
    \leq & \frac{\log n_i}{\eta} + \frac{2 \eta}{3} \cdot \left(B' \cdot (2 \log T)^5 \right)\label{eq:use-bound-l-dl}\\
    \leq & \frac{\log n_i}{\eta} + 32 B' \cdot \log T\label{eq:use-eta-log-ub},
  \end{align}
  where (\ref{eq:use-omwu-local}) follows from Lemma \ref{lem:omwu-local}, (\ref{eq:use-bound-l-dl}) follows from Lemma \ref{lem:bound-l-dl}, and (\ref{eq:use-eta-log-ub}) follows from the upper bound $\eta \leq \frac{1}{Cm \log^4 T}$. We have thus established (\ref{eq:reg-eta-ub-apx}). The upper bound (\ref{eq:reg-eta-fix-apx}) follows immediately.
\end{proof}

\section{Adversarial regret bounds}
In this section we discuss how \Opthedge can be modified to achieve an algorithm that obtains the fast rates of Theorem \ref{thm:polylog-main} when played by all players, and which still obtains the optimal rate of $O(\sqrt{T})$ in the adversarial setting. Such guarantees are common in the literature \citep{daskalakis_near-optimal_2013,rakhlin_optimization_2013,syrgkanis_fast_2015,kangarshahi_honest_2018,hsieh_adaptive_2021}. The guarantees of this modification of \Opthedge are stated in the following corollary (of Lemmas \ref{lem:omwu-local} and \ref{lem:bound-l-dl}):  
\begin{corollary}
  \label{cor:polylog-adv}
  There is an algorithm $\MA$ which, if played by all $m$ players in a game, achieves the regret bound of $\Reg{i}{T} \leq O(m \cdot \log n_i \cdot \log^4 T)$ for each player $i$; moreover, when player $i$ is faced with an \emph{adversarial} sequence of losses, the algorithm $\MA$'s regret bound is $\Reg{i}{T} \leq O(m \log n_i \cdot \log^4 T + \sqrt{T \log n_i})$.
\end{corollary}
\begin{proof}
  Let $C$ be the constant called $C$ in Theorem \ref{thm:polylog-main} and $C'$ be the constant called $C'$ in Lemma \ref{lem:bound-l-dl}. The algorithm $\MA$ of Corollary \ref{cor:polylog-adv} is obtained as follows:
  \begin{enumerate}
  \item Initially run \Opthedge, with the step-size $\eta = \frac{1}{C m\log^4 T}$.
  \item If, for some $T_0 \geq 4$, (\ref{eq:var-l-dl-apx}) first fails to hold at time $T_0$, i.e.,
    \begin{align}
\sum_{t=1}^{T_0} \Var{x_i\^t}{\ell_i\^t - \ell_i\^{t-1}} > \frac 12 \cdot \sum_{t=1}^{T_0} \Var{x_i\^t}{\ell_i\^{t-1}} + C' \lceil \log T \rceil^5\label{eq:change-eta},
    \end{align}
    then set $\eta' = \sqrt{\frac{\log n_i}{T}}$ and continue on running \Opthedge with step size $\eta'$.
  \end{enumerate}
  If there is no $T_0\geq 4$ so that (\ref{eq:change-eta}) holds (and by Lemma \ref{lem:bound-l-dl}, this will be the case when $\MA$ is played by all $m$ players in a game), then the proof of Theorem \ref{thm:polylog-main} shows that the regret of each player $i$ is bounded as $\Reg{i}{T} \leq O(m \log n_i \cdot \log^4 T)$. Otherwise, since $T_0$ is defined as the smallest integer at least 4 so that (\ref{eq:change-eta}) holds, we have
  \begin{align}
\sum_{t=1}^{T_0} \Var{x_i\^t}{\ell_i\^t - \ell_i\^{t-1}} \leq \frac 12 \cdot \sum_{t=1}^{T_0} \Var{x_i\^t}{\ell_i\^{t-1}} + C' \lceil \log T \rceil^5 + 4\nonumber,
  \end{align}
  and thus, by Lemma \ref{lem:omwu-local}, for any $x^\st \in \Delta^{n_i}$,
  \begin{align}
\sum_{t=1}^{T_0} \lng \ell_i\^t, x_i\^t - x^\st \rng \leq \Reg{i}{T_0} \leq O(m \log n_i \cdot \log^4 T_0).\label{eq:before-t0-regret}
  \end{align}
  Further, by the choice of step size $\eta' = \sqrt{\frac{\log n_i}{T}}$ for time steps $t > T_0$, we have, for any $x^\st \in \Delta^{n_i}$,
  \begin{align}
    \sum_{t=T_0+1}^T \lng \ell_i\^t, x_i\^t - x^\st \rng \leq & \frac{\log n_i}{\eta'} + \eta' \sum_{t=T_0+1}^T \| \ell_i\^t - \ell_i\^{t-1} \|_\infty^2\label{eq:use-var-ub-T}\\
    \leq & \frac{\log n_i}{\eta'} + \eta' T \leq O(\sqrt{T \log n_i})\label{eq:after-t0-regret},
  \end{align}
  where (\ref{eq:use-var-ub-T}) uses \cite[Proposition 7]{syrgkanis_fast_2015}. Adding (\ref{eq:before-t0-regret}) and (\ref{eq:after-t0-regret}) completes the proof of the corollary. 
\end{proof}
\end{document}